\documentclass[11pt]{article}

\usepackage[margin=1in]{geometry}

\usepackage{enumitem}
\setlist{nosep,leftmargin=*}

\usepackage{amsmath, amssymb, amsthm}
\usepackage{mathtools}
\usepackage{bm}
\usepackage{bbm}

\usepackage{microtype}
\usepackage{lmodern}
\usepackage[T1]{fontenc}
\usepackage[utf8]{inputenc}

\usepackage{graphicx}
\usepackage{booktabs}
\usepackage{multirow}
\usepackage{array}
\usepackage{caption}
\usepackage{subcaption}

\usepackage{algorithm}
\usepackage{algorithmic}
 
\usepackage{hyperref}
 
\usepackage{natbib}
\setcitestyle{authoryear, open={(}, close={)}}
 
\usepackage{xcolor}
\usepackage{enumitem}
\setlist{noitemsep, topsep=3pt, parsep=0pt, partopsep=0pt}
\usepackage{cleveref}
 
\newtheorem{theorem}{Theorem}[section]
\newtheorem{proposition}[theorem]{Proposition}
\newtheorem{lemma}[theorem]{Lemma}

\newtheorem{assumption}[theorem]{Assumption}
\theoremstyle{remark}
\newtheorem{remark}[theorem]{Remark}
 
\title{\textbf{Decoupled Conformal Optimisation: Efficient Prediction Sets via Independent Tuning and Calibration}}

\author{%
  \begin{tabular}{cccc}
    Fanyi Wu$^{1,2}$ & Lihua Niu$^{1}$ & Samuel Kaski$^{1,3,4}$ & Michele Caprio$^{1}$ \\[4pt]
  \end{tabular}\\[6pt]
  \small $^{1}$Department of Computer Science, University of Manchester, Manchester, UK\\
  \small $^{2}$UKRI AI Centre for Doctoral Training in Decision Making for Complex Systems\\
  \small $^{3}$Department of Computer Science, Aalto University, Espoo, Finland\\
  \small $^{4}$ELLIS Institute, Finland\\
  \small Correspondence: \texttt{fanyi.wu@manchester.ac.uk}
}
\date{}

\begin{document}

\maketitle
\begin{abstract}    
Bayesian conformal optimisation methods often use the same held-out data both to search for efficient prediction sets and to certify coverage or risk. This coupling is natural for high-probability risk-control guarantees, but it is not necessary when the target is standard finite-sample marginal conformal coverage. We propose Decoupled Conformal Optimisation (DCO), a train-tune-calibrate design principle that uses an independent tuning split for efficiency-oriented structural selection and a fresh calibration split for the final conformal quantile. Conditional on the tuned structure, standard split-conformal exchangeability yields finite-sample marginal coverage for any candidate class, without a confidence parameter or multiple-testing correction. DCO therefore targets a different finite-sample guarantee from PAC-style methods: marginal conformal coverage rather than high-probability risk control. Under consistency assumptions on the coupled risk bound, the two approaches nevertheless converge to the same population threshold. Across classification and regression benchmarks, including ImageNet-A, CIFAR-100, Diabetes, California Housing, and Concrete, DCO tracks the nominal coverage level closely while often reducing average prediction-set size or interval width relative to PAC-style calibration. On ImageNet-A, for example, the average set size decreases from $26.52$ to $25.26$ and the 95th-percentile set size from $58.95$ to $53.73$; on Diabetes, the average interval width decreases from $2.098$ to $1.914$.
\end{abstract}
\section{Introduction}
\label{Sec: Introduction}

Reliable uncertainty quantification is a central goal in modern machine learning.
Conformal prediction (CP) offers a distribution-free way to construct prediction
sets with finite-sample marginal coverage under exchangeability
\citep{vovk2005algorithmic, shafer2008tutorial, angelopoulos2021gentle,
barber2023conformal, caprio2025joyscategoricalconformalprediction}. In split CP,
the data are divided into separate roles. A model is fitted on
$D_{\mathrm{train}}$. A non-conformity score $S(x,y)$ is then evaluated on an
independent calibration set $D_{\mathrm{cal}}$. The prediction set is
\begin{equation}
    C(x) = \{y : S(x,y) \leq \hat{q}_{1-\alpha}\},
    \label{eq:split_cp_set}
\end{equation}
where $\hat{q}_{1-\alpha}$ is the empirical $(1-\alpha)$ conformal quantile of
the calibration scores. The final threshold in \eqref{eq:split_cp_set} is an
order statistic. Once the score and model structure have been fixed, this
order statistic is computed on calibration data that are independent of the
test point. This is the structural condition behind the standard
exchangeability proof and the resulting finite-sample marginal coverage
guarantee; see Appendix~\ref{app:splitcp}.

Modern conformal methods increasingly seek more than validity. They also aim for
efficiency: prediction sets should be small, informative, and still valid. In conformal optimisation, efficiency can be improved through the score,
the prior, the surrogate model, or the threshold-search rule
\citep{caprio2024credal, caprio2025conformalized}. These choices introduce a simple but important design question: \textbf{which data should be used for optimisation, and which data should be reserved for calibration?}

The question matters because the two roles are statistically different.
Optimisation searches for an efficient prediction rule. Calibration certifies the
final rule by computing a conformal quantile. If the same held-out split is used
for both roles, the final threshold is no longer computed on data untouched by
the preceding search. The usual split-conformal exchangeability argument
therefore does not apply directly. The problem is not optimisation itself. It is
the reuse of the calibration data for both search and certification.

This distinction is particularly relevant for coupled risk-control procedures,
including Conformal Risk Control (CRC) and PAC-style calibration based on Bayesian quadrature (BQ)
\citep{angelopoulos2025conformalriskcontrol,snell2025conformal}. Such methods select a threshold that satisfies a risk constraint with confidence $1-\delta$. The coverage guarantees derived from PAC-style mthods are different from the marginal coverage guarantee of split CP. PAC-sryle methods like CRC and BQ calibration target high-probability risk control, whereas split CP targets marginal coverage at level $1-\alpha$. When marginal conformal coverage is the desired guarantee, it is natural to ask whether optimisation and calibration need to be coupled on the same data.

We answer this question with \emph{Decoupled Conformal Optimisation} (DCO).
DCO-Warmstart is a train-tune-calibrate design principle for Bayesian conformal
optimisation. It assigns each data split a distinct role. The training split
fits the Bayesian model. The tuning split $D_{\mathrm{tune}}$ selects
efficiency-oriented structure, such as the score, prior, model configuration, or
threshold-search rule. The calibration split $D_{\mathrm{cal}}$ is used only
after this selection step. Its sole purpose is to compute the final conformal
quantile.

This separation restores the split-conformal logic. Conditional on the structure selected using $D_{\mathrm{train}}$ and $D_{\mathrm{tune}}$, the calibration
scores and the test score remain exchangeable. The final threshold is then an order statistic of an untouched calibration set. Hence the standard
split-conformal proof applies without modification, yielding finite-sample
marginal coverage; see Appendix~\ref{app:proofs}. At the same time, DCO-Warmstart still allows explicit optimisation before calibration.

The idea is close in spirit to selecting a model on a validation set before
applying split CP. The validity argument is the classical one. The contribution
is to make the data-separation principle explicit for Bayesian conformal
optimisation pipelines, where efficiency search and coverage assessment are
often intertwined. It also
clarifies the guarantee being targeted: marginal conformal coverage rather than
high-probability risk control.

Our contributions are as follows:
\begin{itemize}

\item \textbf{A decoupled train-tune calibrate principle.}
We formulate DCO-Warmstart as a simple design principle for Bayesian conformal
optimisation. Structural choices are selected on $D_{\mathrm{tune}}$, while the
final conformal quantile is computed on an untouched $D_{\mathrm{cal}}$.

\item \textbf{A finite-sample marginal coverage guarantee.}
We show that DCO-Warmstart inherits the standard split-conformal coverage guarantee once
the tuned structure is fixed independently of $D_{\mathrm{cal}}$. No confidence
parameter $\delta$ is needed. No multiple-testing correction over the candidate
class is required for the final conformal calibration.

\item \textbf{A comparison with CRC/BQ-style calibration.}
We clarify the difference between the finite-sample guarantees of DCO-Warmstart and
CRC/BQ-style methods. DCO-Warmstart targets marginal conformal coverage. CRC/BQ-style
methods target high-probability risk control. We also show that, under uniform
consistency conditions on the risk estimator, both approaches converge to the
same population threshold,
\begin{equation}
    \lambda^\star = \inf\{\lambda : R(\lambda) \leq \alpha\}.
    \label{eq:population_threshold_intro}
\end{equation}

\item \textbf{Empirical evidence across regression and classification.}
We evaluate DCO-Warmstart on ImageNet-A, CIFAR-100, Diabetes, California Housing, and
Concrete, with additional ablations over candidate search, split allocation,
and target coverage level. Across these settings, DCO-Warmstart tracks nominal coverage
closely and often reduces average set size or interval width relative to
BQ/CRC-style calibration.

\end{itemize}

\section{Related Work}
\label{Sec: Related Work}

We organise related work around a central distinction: whether optimisation and 
calibration are performed jointly on the same held-out data, or separated across 
independent splits. This perspective organises classical CP, score-design 
methods, BCP-CRC, LTT, ROCP, and DCO-Warmstart along a common dimension: how optimisation and calibration data are allocated. It also clarifies the statistical role played by each data split.

\paragraph{Quantile-based calibration.}
Classical CP \citep{vovk2005algorithmic, shafer2008tutorial} provides 
distribution-free finite-sample coverage by setting the threshold to an empirical 
quantile of non-conformity scores, without explicit threshold optimisation. 
Efficiency is therefore largely governed by the non-conformity score 
\citep{Bellotti, dhillon2024expectedsizeconformalprediction, Sadinle_2018}. 
A substantial literature improves efficiency by reshaping the score rather than the final conformal threshold. Examples include adaptive classification scores such as RAPS \citep{angelopoulos2021uncertainty}, regression-adapted residuals 
\citep{lei2018distribution}, and posterior predictive densities via AOI 
importance reweighting in Bayesian settings \citep{fong2021conformal}. Since 
score design and conformal quantile calibration play distinct statistical roles, 
DCO-Warmstart is compatible with scores from this line of work.

\paragraph{Risk-constrained threshold optimisation.}
A second line of work treats the threshold $\lambda$ as a decision variable. 
BCP-CRC \citep{wu2026bayesian}, for example, selects $\lambda$ by minimising 
expected prediction set size subject to a high-probability miscoverage constraint 
enforced through the $L^+$ bound of CRC 
\citep{angelopoulos2025conformalriskcontrol, snell2025conformal}:
\begin{equation}
    \begin{aligned}
    \min_{\lambda}\quad 
    & \mathbb{E}_X\big[|C(X;\lambda)|\big] \\
    \text{s.t.}\quad 
    & \mathbb{P}\!\left(
    \mathbb{P}\big(Y \notin C(X;\lambda)\big) \leq \alpha
    \right) \geq 1-\delta .
    \end{aligned}
    \label{eq:crc_risk_control}
\end{equation}
Here the inner probability is the population miscoverage risk for a fixed 
threshold, while the outer probability is taken over the draw of the calibration 
sample. Thus, \eqref{eq:crc_risk_control} gives a high-probability risk-control 
statement over the calibration sample, rather than the standard marginal 
coverage guarantee of split CP 
\citep{vovk2005algorithmic, shafer2008tutorial}. Such a guarantee is useful when the objective is risk certification, especially with limited calibration data or frequent recalibration. The trade-off is that threshold selection and risk 
certification are performed on the same held-out data, making BCP-CRC the closest 
coupled baseline to DCO-Warmstart.

DCO-Warmstart is not a replacement for CRC/BQ-style methods when the scientific objective 
is high-probability risk control. The procedures target different statistical 
guarantees. CRC/BQ-style methods are appropriate when one wants a risk certificate that holds with confidence $1-\delta$ over the calibration sample. DCO-Warmstart is appropriate when the target guarantee is the standard finite-sample marginal 
coverage guarantee of split CP. Our claim is therefore not 
that coupled calibration is unnecessary in general, but that it is unnecessary 
for marginal conformal coverage when optimisation can be performed on an 
independent tuning split.

Learn-then-Test (LTT) \citep{angelopoulos2025ltt} also operates in a 
population-risk regime, but calibrates feasibility through hypothesis testing 
rather than constrained optimisation. Like CRC/BQ-style methods, it targets 
high-probability risk control rather than marginal split-conformal coverage.

\paragraph{Decision-theoretic set optimisation.}
Risk-Optimal Conformal Prediction (ROCP) 
\citep{wang2026optimaldecisionmakingbasedprediction} optimises the full 
prediction-set construction and downstream action rules for decision quality, and then applies CP on independent data to restore marginal coverage. Like DCO-Warmstart, ROCP separates the optimisation stage from the final conformal calibration step. The 
difference is scope: ROCP intervenes at the level of the entire set-construction 
and action rule, whereas DCO-Warmstart intervenes at the level of Bayesian structural and 
threshold-search configuration.

\paragraph{Where DCO-Warmstart fits.}
The methods above differ along two dimensions: what is optimised and when 
certification occurs. Classical CP certifies directly by an empirical quantile 
without threshold optimisation. Score-design methods optimise the score before 
calibration and then certify by quantile. BCP-CRC and related CRC/BQ-style 
methods optimise and certify risk on the same held-out data. LTT certifies 
population-risk feasibility through testing. ROCP optimises set structure and 
then calibrates on independent data.

DCO-Warmstart occupies a complementary position. Structural choices are selected on an 
independent tuning split, while the final conformal quantile is computed on a 
separate calibration split. This recovers the standard finite-sample marginal 
coverage guarantee without an additional confidence parameter. At the same time, 
under consistency assumptions on the coupled risk bound, DCO-Warmstart remains 
asymptotically aligned with CRC/BQ-style methods at the level of the population 
threshold. This asymptotic alignment should not be interpreted as an equivalence 
of finite-sample guarantees: DCO-Warmstart targets marginal coverage, whereas CRC/BQ-style 
methods target high-probability risk control. The practical consequence is a 
difference in finite-sample guarantee type, which we examine empirically in 
Section~\ref{sec:experiment}.

\begin{table}[t]
\centering
\caption{Comparison of optimisation and calibration roles across methods. 
For methods with a confidence parameter, $\delta$ denotes the failure probability; 
the guarantee holds with confidence $1-\delta$ over the calibration sample.}
\label{tab:method_comparison}
\setlength{\tabcolsep}{4pt}
\renewcommand{\arraystretch}{1.05}
\resizebox{\textwidth}{!}{%
\begin{tabular}{lcccc}
\toprule
Method 
& Optimisation data 
& Calibration data 
& Guarantee 
& Confidence parameter \\
\midrule
Split CP 
& none / fixed score 
& $D_{\mathrm{cal}}$ 
& Marginal coverage 
& No \\
Score-tuned CP 
& $D_{\mathrm{tune}}$ 
& $D_{\mathrm{cal}}$ 
& Marginal coverage 
& No \\
CRC/BQ-style 
& $D_{\mathrm{cal}}$ 
& $D_{\mathrm{cal}}$ 
& High-probability risk control 
& Yes, $\delta$ \\
LTT 
& $D_{\mathrm{val}}$ (via testing) 
& same as optimisation 
& High-probability risk control 
& Yes, $\delta$ \\
ROCP 
& $D_{\mathrm{opt}}$ 
& $D_{\mathrm{cal}}$ 
& Marginal coverage 
& No \\
DCO-Warmstart 
& $D_{\mathrm{tune}}$ 
& $D_{\mathrm{cal}}$ 
& Marginal coverage 
& No \\
DirectTune 
& $D_{\mathrm{tune}}$ 
& none 
& None in general 
& No \\
\bottomrule
\end{tabular}%
}
\end{table}
\section{Theoretical Background}
\label{Sec: Theory}

We formalise the prediction problem and establish the 
theoretical properties of DCO-Warmstart. All proofs are deferred 
to Appendix~\ref{app:proofs}.

\subsection{Problem Setup}

Let $(X,Y)\in\mathcal{X}\times\mathcal{Y}$ be drawn 
i.i.d.\ from an unknown distribution $P$. Prediction 
sets are parameterised by structural choices 
$\phi\in\Phi$, such as the score function type, prior 
hyperparameters, or model architecture, together with 
a scalar threshold $\lambda\in\Lambda$:
\begin{equation}
    C_{\phi,\lambda}(x)=\{y:S_\phi(x,y)\le\lambda\},
\end{equation}
where $S_\phi$ is a non-conformity score derived from 
the posterior predictive distribution 
$p(y\mid x,D_{\text{train}})$. Larger $\lambda$ 
produces larger prediction sets; we assume 
$\lambda_1\le\lambda_2$ implies 
$C_{\phi,\lambda_1}(x)\subseteq C_{\phi,\lambda_2}(x)$. 
The two quantities of interest are the population 
miscoverage risk and expected set size,
\begin{equation}
    R(\phi,\lambda)=\mathbb{P}(Y\notin C_{\phi,\lambda}(X)),
    \qquad
    \mathcal{S}(\phi,\lambda)=
    \mathbb{E}[|C_{\phi,\lambda}(X)|],
\end{equation}
and the data are partitioned into three independent 
splits: $D_{\text{train}}$ for model fitting, 
$D_{\text{tune}}$ for structural optimisation, and 
$D_{\text{cal}}$ for conformal calibration.

\subsection{Finite-sample marginal
coverage of DCO-Warmstart}

DCO-Warmstart selects structural choices on $D_{\text{tune}}$ 
by solving the empirical problem
\begin{equation}
    (\hat{\phi}_{\text{tune}},\hat{\lambda}_{\text{tune}})
    \in\operatorname*{arg\,min}_{(\phi,\lambda)\in\Phi\times\Lambda}
    \widehat{\mathcal{S}}_{\text{tune}}(\phi,\lambda)
    \quad\text{s.t.}\quad
    \widehat{R}_{\text{tune}}(\phi,\lambda)\le\alpha,
    \label{eq:dco-tuning}
\end{equation}
where $\widehat{R}_{\text{tune}}$ and 
$\widehat{\mathcal{S}}_{\text{tune}}$ are empirical 
estimates on $D_{\text{tune}}$. Only 
$\hat{\phi}_{\text{tune}}$ is carried forward. The
tuning threshold $\hat{\lambda}_{\text{tune}}$ is used
only to rank candidate structures on
$D_{\text{tune}}$ and is discarded before deployment.
The deployed threshold is the split-conformal
quantile computed on the independent calibration
split $D_{\text{cal}}$. Since $\hat{\phi}_{\text{tune}}$ 
does not depend on $D_{\text{cal}}$, the standard 
exchangeability argument of split CP applies directly.

\begin{theorem}[Finite-sample marginal coverage of DCO-Warmstart]
\label{thm:dco-coverage}
Assume that
\[
    (X_1,Y_1),\ldots,(X_m,Y_m),(X_{m+1},Y_{m+1})
\]
are exchangeable conditional on $D_{\mathrm{train}}$ and
$D_{\mathrm{tune}}$, and that $\hat{\phi}_{\textup{tune}}$
is measurable with respect to
$D_{\mathrm{train}}\cup D_{\mathrm{tune}}$ only. Let
\[
    S_i=S_{\hat{\phi}_{\textup{tune}}}(X_i,Y_i),
    \qquad i=1,\ldots,m,
\]
let $S_{(1)}\le\cdots\le S_{(m)}$ denote the sorted
calibration scores, and define
\[
    k_\alpha=\left\lceil (m+1)(1-\alpha)\right\rceil,
    \qquad
    \hat{q}_{1-\alpha}=
    \begin{cases}
        S_{(k_\alpha)}, & k_\alpha\le m,\\
        +\infty, & k_\alpha=m+1.
    \end{cases}
\]
Then
\[
    \mathbb{P}\!\left\{
    Y_{m+1}\in
    C_{\hat{\phi}_{\textup{tune}},\hat{q}_{1-\alpha}}
    (X_{m+1})
    \right\}
    \ge 1-\alpha.
\]
This guarantee holds for any candidate class $\Phi$,
finite or infinite, because calibration is applied only
after a single tuned structure has been fixed
independently of $D_{\mathrm{cal}}$.
\end{theorem}


\subsection{Sample Complexity}

We do not claim an end-to-end finite-sample oracle guarantee for the final recalibrated DCO-Warmstart predictor. Instead, we give a finite-class uniform-convergence result for the tuning stage alone, which explains when the tuning split is large enough to select an efficient candidate structure before the independent conformal calibration step is applied.

\begin{proposition}[Tuning oracle inequality for a finite search class]
\label{prop:tuning_oracle}
Let $\mathcal{A}$ be a finite class of candidate prediction-set rules, where each $a\in\mathcal{A}$ defines a set-valued predictor $C_a$. Let
\[
    R(a)=\mathbb{P}\{Y\notin C_a(X)\},
    \qquad
    S(a)=\mathbb{E}[s(C_a(X))]
\]
denote its miscoverage risk and expected size, where $s(C_a(X))\in[0,B]$. Let $\widehat{R}_{\mathrm{tune}}(a)$ and $\widehat{S}_{\mathrm{tune}}(a)$ be the corresponding empirical quantities on $m_{\mathrm{tune}}$ independent tuning samples. Fix $\varepsilon_R,\varepsilon_S,\eta>0$. With probability at least $1-\eta$, uniformly over $a\in\mathcal{A}$,
\[
    \left|\widehat{R}_{\mathrm{tune}}(a)-R(a)\right|
    \leq \varepsilon_R,
    \qquad
    \left|\widehat{S}_{\mathrm{tune}}(a)-S(a)\right|
    \leq \varepsilon_S,
\]
provided
\[
    m_{\mathrm{tune}}
    \geq
    \max\left\{
    \frac{\log(4|\mathcal{A}|/\eta)}{2\varepsilon_R^2},
    \frac{B^2\log(4|\mathcal{A}|/\eta)}{2\varepsilon_S^2}
    \right\}.
\]
Consequently, if the tuning rule selects
\[
    \widehat{a}
    \in
    \arg\min_{a\in\mathcal{A}}
    \widehat{S}_{\mathrm{tune}}(a)
    \quad
    \text{s.t.}
    \quad
    \widehat{R}_{\mathrm{tune}}(a)\leq \alpha-\varepsilon_R,
\]
then, on the same event
\[
    R(\widehat{a})\leq \alpha
\]
and
\[
    S(\widehat{a})
    \leq
    \inf_{a\in\mathcal{A}:R(a)\leq \alpha-2\varepsilon_R}
    S(a)
    +
    2\varepsilon_S.
\]
\end{proposition}

\paragraph{Interpretation.}
The final deployed DCO-Warmstart set is still calibrated on $D_{\mathrm{cal}}$, so its finite-sample marginal coverage does not rely on the empirical feasibility event in the proposition. The proposition instead explains when the tuning split is large enough to select an efficient structure before the independent conformal calibration step is applied.
Calibration accuracy can be analysed separately once the tuned structure is fixed. Appendix~\ref{app:proofs} gives a Dvoretzky--Kiefer--Wolfowitz-based lemma showing that, under local regularity around the population quantile, the empirical conformal quantile concentrates around its population target at rate $m_{\mathrm{cal}}^{-1/2}$.

\subsection{Asymptotic Agreement with CRC/BQ}

\begin{proposition}[Asymptotic agreement under uniform risk-bound consistency]
\label{prop:asymptotic}
Fix a structure $\phi$ and write
\[
    R(\lambda)=\mathbb{P}\{Y\notin C_{\phi,\lambda}(X)\}.
\]
Assume:
\begin{enumerate}
    \item $R(\lambda)$ is continuous and strictly decreasing in a neighbourhood of
    \[
        \lambda^\star = \inf\{\lambda:R(\lambda)\leq\alpha\}.
    \]
    \item The split-conformal DCO-Warmstart threshold 
    satisfies
    \[
        \widehat{\lambda}_{\mathrm{DCO}}\xrightarrow{p}\lambda^\star.
    \]
    \item The coupled CRC/BQ threshold can be written as
    \[
        \widehat{\lambda}_{\mathrm{CRC}}
        =
        \inf\{\lambda:
        \widehat{R}_m(\lambda)+b_m(\lambda,\delta_m)\leq\alpha\},
    \]
    where
    \[
        \sup_{\lambda\in\Lambda}
        |\widehat{R}_m(\lambda)-R(\lambda)|
        \xrightarrow{p}0,
        \qquad
        \sup_{\lambda\in\Lambda}
        b_m(\lambda,\delta_m)
        \xrightarrow{p}0.
    \]
\end{enumerate}
Then
\[
    \widehat{\lambda}_{\mathrm{CRC}}
    -
    \widehat{\lambda}_{\mathrm{DCO}}
    \xrightarrow{p}0.
\]
\end{proposition}

Assumption~2 is the standard consistency requirement for the split-conformal
quantile. It holds under mild local regularity of the score distribution; in
particular, Lemma~\ref{lem:cal-accuracy} in Appendix~\ref{app:proofs} gives a DKW-based concentration bound for
$\hat{q}_{1-\alpha}$ around its population quantile. Thus the proposition should
be read as a comparison of large-sample targets, not as an equivalence of
finite-sample guarantees.

\paragraph{Interpretation.}
DCO-Warmstart gives finite-sample marginal conformal coverage, whereas CRC/BQ-style
methods give high-probability risk control. Proposition~\ref{prop:asymptotic}
only states that, when the coupled risk bound consistently estimates the
population risk boundary and its excess margin vanishes, the selected thresholds
approach the same population limit.

\section{Decoupled Conformal Optimisation}
\label{Sec: Method}

\begin{figure}[h]
    \centering
    \includegraphics[width=0.85\columnwidth]{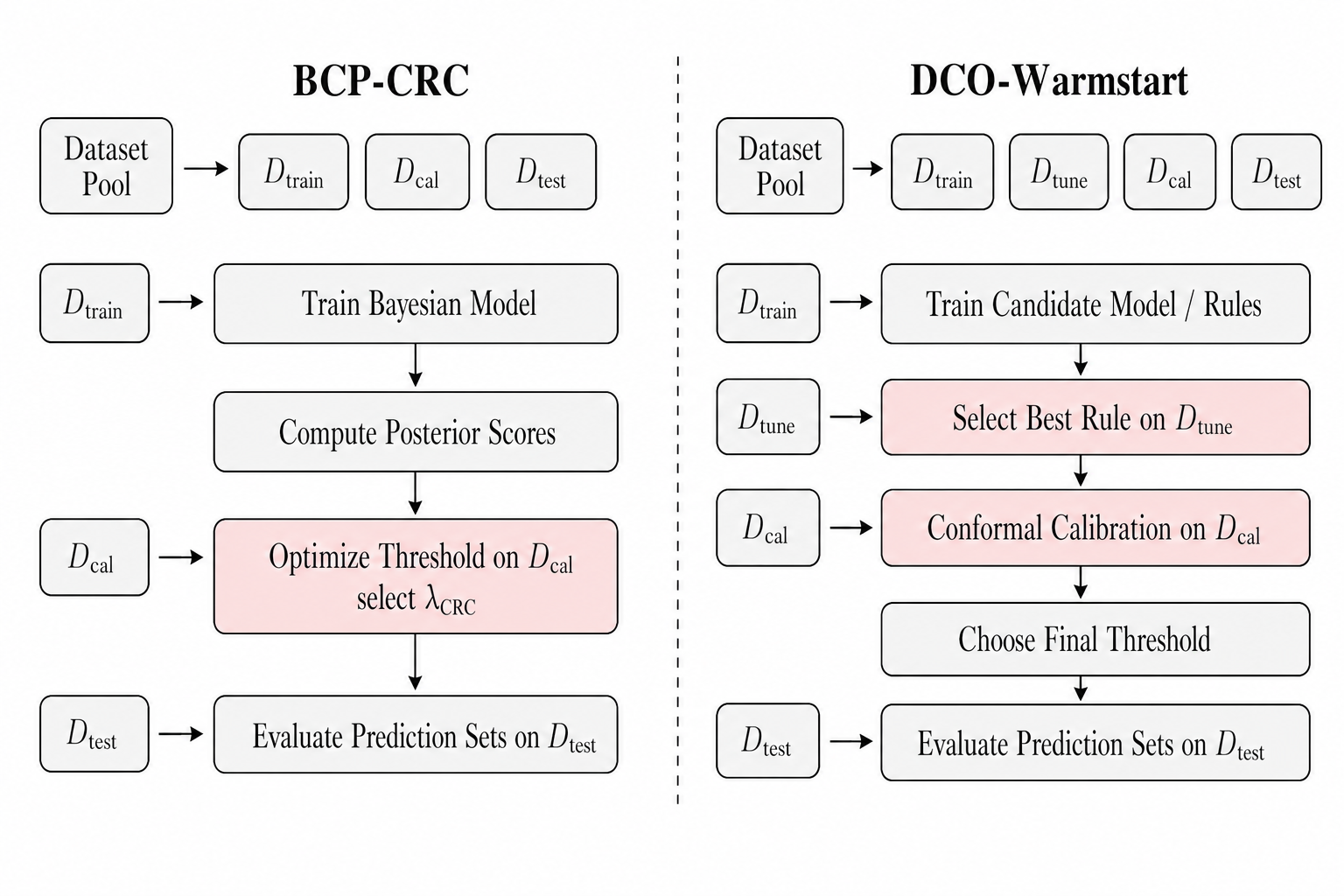}
    \caption{
    Coupled calibration versus DCO-Warmstart. In CRC/BQ-style calibration, the same calibration split is
    used both to search for an efficient threshold and to certify risk. DCO-Warmstart separates these roles:
    $D_{\mathrm{tune}}$ is used for score/model/hyperparameter selection, while
    $D_{\mathrm{cal}}$ is reserved exclusively for the final conformal quantile. This separation
    is the key condition that allows the standard split-conformal exchangeability argument to
    apply after tuning.
    }
    \label{fig:pipeline}
\end{figure}
Building on Section~\ref{Sec: Theory}, we describe the 
operational procedure for DCO-Warmstart. As shown in 
Figure~\ref{fig:pipeline}, the key departure from BCP 
is the introduction of a dedicated tuning split 
$D_{\text{tune}}$: structural selection and conformal 
calibration are allocated to independent data splits, 
so the exchangeability argument of 
Theorem~\ref{thm:dco-coverage} applies without 
modification.

\subsection{Data Splitting and Tuning}
Given exchangeable data $D_n$, we partition it into 
three disjoint splits $D_{\text{train}}$, 
$D_{\text{tune}}$, and $D_{\text{cal}}$, as defined in 
Section~\ref{Sec: Theory}. The posterior 
$\pi(\theta\mid D_{\text{train}})$ is estimated on 
$D_{\text{train}}$, thereby fixing the score function 
$S_\phi(x,y)$ with respect to all subsequent splits. 
In the BCP setting, $S_\phi(x,y) = -\log p(y\mid x, 
D_{\text{train}})$ is the negative log posterior 
predictive density, and $\phi$ encodes structural 
choices such as the score type, prior hyperparameters, 
or model architecture.

On $D_{\text{tune}}$, we compute the empirical 
miscoverage and average set size,
\begin{align}
    \widehat{R}_{\text{tune}}(\phi,\lambda)
    &=
    \frac{1}{|D_{\text{tune}}|}
    \sum_{(X_i,Y_i)\in D_{\text{tune}}}
    \mathbf{1}\{Y_i\notin C_{\phi,\lambda}(X_i)\},
    \label{eq:emp-risk}\\
    \widehat{\mathcal{S}}_{\text{tune}}(\phi,\lambda)
    &=
    \frac{1}{|D_{\text{tune}}|}
    \sum_{(X_i,Y_i)\in D_{\text{tune}}}
    |C_{\phi,\lambda}(X_i)|,
    \label{eq:emp-size}
\end{align}
where $\lambda$ denotes a candidate threshold used to 
form the tentative prediction set $C_{\phi,\lambda}(x)$ 
during tuning. We then select structural choices by solving
\begin{equation}
    (\hat{\phi}_{\text{tune}},\,\hat{\lambda}_{\text{tune}})
    =
    \operatorname*{arg\,min}_{(\phi,\lambda)\in\Phi\times\Lambda}
    \widehat{\mathcal{S}}_{\text{tune}}(\phi,\lambda)
    \quad\text{s.t.}\quad
    \widehat{R}_{\text{tune}}(\phi,\lambda)\le\alpha.
    \label{eq:dco-tuning-empirical}
\end{equation}
In practice, \eqref{eq:dco-tuning-empirical} is solved 
by grid search over $\Phi\times\Lambda$; the monotonicity of 
$\widehat{R}_{\text{tune}}(\phi,\cdot)$ in $\lambda$ 
permits an efficient line search for each fixed $\phi$. 
If no candidate pair satisfies the constraint 
$\widehat{R}_{\text{tune}}(\phi,\lambda)\le\alpha$, 
we select the candidate with the smallest empirical 
miscoverage, breaking ties in favor of smaller average 
set size.

It is important to note the distinct roles of the two 
outputs of \eqref{eq:dco-tuning-empirical}. The 
selected structure $\hat{\phi}_{\text{tune}}$ is 
carried forward as the chosen model configuration. The 
accompanying threshold $\hat{\lambda}_{\text{tune}}$, 
however, serves purely as a ranking device during the 
search over $\Phi$: it identifies how tight a threshold 
is needed for each candidate $\phi$ to satisfy the 
empirical miscoverage constraint on $D_{\text{tune}}$, 
thereby ranking candidates by their empirical efficiency. 
The threshold $\hat{\lambda}_{\text{tune}}$ is 
\emph{not} deployed. Instead, the deployed threshold is 
the conformal quantile
\begin{equation}
    \hat{q}_{1-\alpha}
    =
    \mathrm{Quantile}\!\left(
        \{S_{\hat{\phi}_{\text{tune}}}(X_i,Y_i)\}_{(X_i,Y_i)\in D_{\text{cal}}},
        \,\frac{\lceil(1-\alpha)(|D_{\text{cal}}|+1)\rceil}{|D_{\text{cal}}|}
    \right),
    \label{eq:conformal-quantile}
\end{equation}
computed afresh on the held-out calibration set 
$D_{\text{cal}}$. Because $\hat{\phi}_{\text{tune}}$ is 
determined using only $D_{\text{train}}$ and 
$D_{\text{tune}}$, and is not adapted to 
$D_{\text{cal}}$, the marginal coverage guarantee of 
Theorem~\ref{thm:dco-coverage} applies to the prediction 
set $C_{\hat{\phi}_{\text{tune}},\,\hat{q}_{1-\alpha}}$. 
We refer to this two-stage procedure--tuning over 
$(\phi,\lambda)$ pairs on $D_{\text{tune}}$ to select 
the structure, then recalibrating on $D_{\text{cal}}$ 
to obtain the deployed threshold--as 
\textbf{DCO-Warmstart}.

\subsection{DCO-Warmstart: Structure Selection with 
Conformal Calibration}

Once the structure $\hat{\phi}_{\text{tune}}$ has been 
selected on $D_{\text{tune}}$, the deployed threshold is 
determined entirely by the held-out calibration set 
$D_{\text{cal}}$. Let 
$D_{\mathrm{cal}}=\{(X_i,Y_i)\}_{i=1}^{m}$, define the 
calibration scores
\[
    S_i = S_{\hat{\phi}_{\text{tune}}}(X_i,Y_i),
    \qquad i=1,\ldots,m,
\]
and let 
\[
    S_{(1)}\le \cdots \le S_{(m)}
\]
denote the sorted calibration scores. Define
\[
    k_\alpha = \left\lceil (m+1)(1-\alpha)\right\rceil .
\]
The exact split-conformal calibration threshold is then
\begin{equation}
    \hat{q}_{1-\alpha}
    =
    \begin{cases}
        S_{(k_\alpha)}, & k_\alpha \le m,\\
        +\infty, & k_\alpha = m+1.
    \end{cases}
    \label{eq:dco-exact-quantile}
\end{equation}

The resulting prediction set is
\begin{equation}
    C_{\text{DCO}}(x)
    =
    \left\{
    y:
    S_{\hat{\phi}_{\text{tune}}}(x,y)
    \le
    \hat{q}_{1-\alpha}
    \right\}.
    \label{eq:dco-set}
\end{equation}
By Theorem~\ref{thm:dco-coverage}, this prediction set 
satisfies finite-sample marginal coverage. This is the 
primary certified procedure used throughout the paper. 
At deployment, the method uses only the structure 
selected on $D_{\mathrm{tune}}$ together with the 
conformal calibration threshold computed on 
$D_{\mathrm{cal}}$; all tuning thresholds, including 
$\hat{\lambda}_{\text{tune}}$, are discarded.

\begin{algorithm}[htbp]
\caption{DCO-Warmstart: Structure Selection with Conformal Calibration}
\label{alg:dco-v1}
\begin{algorithmic}[1]
\REQUIRE Data $D_n$, miscoverage level $\alpha$, 
         candidate structure class $\Phi$, threshold grid $\Lambda$
\STATE Partition $D_n$ into $D_{\text{train}}$, 
       $D_{\text{tune}}$, and $D_{\text{cal}}$
\STATE Estimate posterior $\pi(\theta\mid D_{\text{train}})$
       and define scores $\{S_\phi\}_{\phi\in\Phi}$
\STATE Solve~\eqref{eq:dco-tuning-empirical} on 
       $D_{\text{tune}}$ to obtain 
       $(\hat{\phi}_{\text{tune}},\hat{\lambda}_{\text{tune}})$
\STATE Discard $\hat{\lambda}_{\text{tune}}$ and retain only 
       $\hat{\phi}_{\text{tune}}$
\STATE Compute 
       $S_i=S_{\hat{\phi}_{\text{tune}}}(X_i,Y_i)$
       for each $(X_i,Y_i)\in D_{\text{cal}}$
\STATE Sort $S_1,\ldots,S_m$ into 
       $S_{(1)}\le\cdots\le S_{(m)}$
\STATE Set $k_\alpha=\lceil(m+1)(1-\alpha)\rceil$
\STATE Set 
       \[
       \hat{q}_{1-\alpha}
       =
       \begin{cases}
       S_{(k_\alpha)}, & k_\alpha\le m,\\
       +\infty, & k_\alpha=m+1
       \end{cases}
       \]
\ENSURE 
       $C_{\text{DCO}}(x)=
       \{y:
       S_{\hat{\phi}_{\text{tune}}}(x,y)
       \le\hat{q}_{1-\alpha}\}$
\end{algorithmic}
\end{algorithm}

\subsection{DirectTune}

As a diagnostic baseline, we also consider a direct 
threshold-tuning procedure. Unlike DCO-Warmstart, 
DirectTune does not perform structure selection followed 
by recalibration. Instead, for a fixed externally chosen 
structure $\phi_0\in\Phi$ and its corresponding score 
function $S_{\phi_0}$, it optimizes the threshold on 
$D_{\text{tune}}$ and deploys that threshold directly. 
Specifically, DirectTune selects
\begin{equation}
    \hat{\lambda}_{\text{tune}}
    \in
    \operatorname*{arg\,min}_{\lambda\in\Lambda}
    \widehat{\mathcal{S}}_{\text{tune}}(\phi_0,\lambda)
    \quad\text{s.t.}\quad
    \widehat{R}_{\text{tune}}(\phi_0,\lambda)\le\alpha,
    \label{eq:dco-v2}
\end{equation}
and deploys the prediction set
\begin{equation}
    C_{\text{DirectTune}}(x)
    =
    \{y:S_{\phi_0}(x,y)\le\hat{\lambda}_{\text{tune}}\}.
    \label{eq:direct-tune-set}
\end{equation}

DirectTune is not conformally certified. Because the 
threshold is selected using the same data on which 
feasibility is evaluated, the empirical constraint on 
$D_{\text{tune}}$ does not imply finite-sample marginal 
coverage for future test points. We therefore use 
DirectTune only as a diagnostic baseline to quantify the 
cost of omitting the final calibration step.

\begin{algorithm}[htbp]
\caption{DirectTune}
\label{alg:dco-v2}
\begin{algorithmic}[1]
\REQUIRE Data $D_n$, miscoverage level $\alpha$, 
         fixed structure $\phi_0$, score function $S_{\phi_0}$, 
         threshold grid $\Lambda$
\STATE Partition $D_n$ into $D_{\text{train}}$ and 
       $D_{\text{tune}}$ 
       \COMMENT{No calibration split; hence no conformal guarantee}
\STATE Estimate posterior $\pi(\theta\mid D_{\text{train}})$
       and fix the score function $S_{\phi_0}$
\STATE Solve~\eqref{eq:dco-v2} on $D_{\text{tune}}$ 
       to obtain $\hat{\lambda}_{\text{tune}}$
\ENSURE 
       $C_{\text{DirectTune}}(x)=
       \{y:S_{\phi_0}(x,y)\le\hat{\lambda}_{\text{tune}}\}$
\end{algorithmic}
\end{algorithm}

\subsection{Computational Complexity}

In practice, DCO-Warmstart scales linearly with the number of 
candidate structures $K=|\Phi|$. For a calibration 
split of size $m_{\text{cal}}$, the final conformal 
calibration step requires 
$\mathcal{O}(m_{\text{cal}}\log m_{\text{cal}})$ 
operations due to sorting the calibration scores. The 
tuning cost depends on the threshold search strategy. 
Assuming that the per-point score evaluation cost is 
$\mathcal{O}(1)$, a direct grid search over 
$\Phi\times\Lambda$ requires 
$\mathcal{O}(K|\Lambda|m_{\text{tune}})$ operations. 
The monotonicity of 
$\widehat{R}_{\text{tune}}(\phi,\cdot)$ in $\lambda$ 
can reduce this cost by permitting an efficient line 
search for each fixed $\phi$. In the BCP setting, the 
cost of evaluating the posterior predictive density 
$p(y\mid x,D_{\text{train}})$ may dominate this 
bookkeeping cost, depending on the posterior 
approximation and the number of posterior samples used.

For a finite candidate class, the tuning-stage oracle 
inequality in Proposition~\ref{prop:tuning_oracle} 
requires
\begin{equation}
    m_{\mathrm{tune}}
    =
    \Omega\!\left(
    \max\left\{
    \frac{\log(K/\eta)}{\varepsilon_R^2},
    \frac{B^2\log(K/\eta)}{\varepsilon_S^2}
    \right\}
    \right),
    \label{eq:tuning-sample-complexity}
\end{equation}
for uniform control of empirical miscoverage and 
empirical size across candidates, where 
$\varepsilon_R$ and $\varepsilon_S$ denote the desired 
uniform deviations for miscoverage and size, $B$ bounds 
the prediction-set size functional, and $\eta$ is the 
failure probability. This tuning-stage result controls 
the quality of structure selection. The final 
finite-sample marginal coverage guarantee is supplied 
separately by Theorem~\ref{thm:dco-coverage} through the 
independent calibration split $D_{\mathrm{cal}}$.
\section{Experiments}
\label{sec:experiment}

We evaluate DCO-Warmstart on regression and classification tasks. We use \emph{certified} 
to denote methods equipped with a formal risk or coverage guarantee under their 
respective calibration procedures (e.g., conformal marginal coverage for DCO-Warmstart and 
high-probability risk control guarantees for BQ). Throughout, we use $\delta$ to denote 
the failure probability and $1-\delta$ the confidence level for risk-control 
methods; BQ is run with $\delta = 0.05$, corresponding to confidence $0.95$. 
Where prior BQ literature uses $\beta$ for the confidence level, we set 
$\beta = 1-\delta$ to align notation.

Full model specification and implementation details are in Appendix~\ref{app:exp}. DCO-Warmstart serves as the primary certified method, DirectTune as a diagnostic baseline, and BQ \citep{snell2025conformal} as the closest methodological comparator; 
Split~CP and CQR serve as standard predictive baselines. DCO-Warmstart and BQ are evaluated over 50 random splits using each method's own calibration protocol; statistical reliability is assessed via paired Wilcoxon signed-rank tests.
\paragraph{Matched comparison protocol.}
To isolate the effect of decoupling from the effect of candidate search, we report several matched-budget controls. Since BQ/CRC-style calibration does not natively perform the same structural search as DCO-Warmstart, we separate three comparisons. First, we compare fixed-structure DCO-Warmstart and fixed-structure BQ/CRC using the same score/model configuration. Second, we select a structure using DCO-Warmstart on $D_{\mathrm{tune}}$ and then recalibrate that fixed structure using the BQ/CRC risk-control protocol on its combined calibration pool; this isolates the calibration mechanism after holding the selected structure fixed. Third, we report an exploratory matched-$\Phi$ BQ/CRC extension, in which BQ/CRC is evaluated over the same candidate class used by DCO-Warmstart. These controls distinguish candidate search, calibration mechanism, and data-budget allocation.

\paragraph{Experimental reporting.}
For each dataset and target coverage level, we report empirical coverage\footnote{Empirical coverage on a finite test split can fall slightly below the nominal level even when the procedure satisfies a finite-sample marginal coverage guarantee. The guarantee concerns the probability over future exchangeable test points and data splits, not deterministic coverage on every realised finite test set.}, average prediction-set size or interval width, and the 95th percentile of set size or interval width. All results are averaged over repeated random splits. When comparing DCO-Warmstart and BQ/CRC on the same splits, we report paired Wilcoxon signed-rank tests for size and coverage differences. Since DCO-Warmstart and BQ/CRC target different guarantees, we interpret these tests descriptively rather than as evidence that one guarantee dominates the other.

\subsection{Regression: Diabetes Dataset}

\paragraph{Setup.}
We use the Diabetes dataset ($n=442$, $d=10$) with 
target coverage $1-\alpha=0.8$. Each run partitions 
the data into four disjoint splits with approximate 
sizes $|D_{\text{train}}|\approx150$, 
$|D_{\text{tune}}|\approx112$, 
$|D_{\text{cal}}|\approx113$, 
$|D_{\text{test}}|\approx67$.
A sparse Bayesian linear regression model is fitted 
on $D_{\text{train}}$ via NUTS MCMC with $T=8{,}000$ 
posterior samples; Split~CP and CQR are included as 
standard predictive baselines alongside BQ. Full specifications are in 
Appendix~\ref{app:exp-regression}. 

\paragraph{DCO-Warmstart pipeline.}
DCO-Warmstart proceeds in two stages on $D_{\text{tune}}$,
independent of $D_{\text{cal}}$.
First, the prior scale $c \in \{1.0,\,0.02\}$ is 
selected by evaluating empirical coverage and 
interval width; $c=1.0$ is chosen as the most 
efficient feasible option in 28 of 50 splits.
Second, the scalar threshold $\lambda$ is optimised 
by minimising average interval width subject to 
empirical coverage $\geq 1-\alpha$. The selected 
$\hat{\lambda}_{\text{tune}}$ is then conformally 
recalibrated on $D_{\text{cal}}$, yielding certified 
DCO-Warmstart intervals; DirectTune applies $\hat{\lambda}_{\text{tune}}$ directly without recalibration.

\paragraph{Results.}
Table~\ref{tab:regression_main} summarises coverage 
and interval width over 50 splits. Across methods, 
BQ is systematically more conservative than 
Split~CP in both coverage and interval width 
(Figure~\ref{fig:aoi_splitcp}), reflecting the 
additional margin introduced by its high-probability 
calibration criterion. DCO-Warmstart achieves empirical coverage close 
to the target ($0.805\pm0.066$) and produces among 
the narrowest certified intervals ($1.914\pm0.192$), 
with both differences relative to BQ statistically 
significant (paired Wilcoxon; see caption). Direct 
tuning matches DCO-Warmstart's average width but without a 
coverage guarantee, exhibiting higher per-split 
variance consistent with Remark~\ref{rem:version2-app}. 
Figure~\ref{fig:regression_results} further shows 
that DCO-Warmstart concentrates tightly around the nominal 
level, sitting between the underconservative direct 
tuning and the overconservative BQ.

\begin{table}[h]
\centering
\caption{Regression results on the Diabetes dataset 
over 50 random splits (target 
$1-\alpha=0.8$). Paired Wilcoxon $p$-values 
(DCO-Warmstart vs.\ BQ, $n=50$): width 
$p=3.09\times10^{-12}$, coverage 
$p=2.36\times10^{-7}$.}
\label{tab:regression_main}
\resizebox{0.65\linewidth}{!}{%
\begin{tabular}{lccc}
\toprule
Method & Coverage & Avg.\ Interval Width & Certified \\
\midrule
BQ
    & $0.842 \pm 0.051$
    & $2.098 \pm 0.131$
    & \checkmark \\
Split~CP
    & $0.812 \pm 0.057$
    & $1.920 \pm 0.159$
    & \checkmark \\
CQR
    & $0.807 \pm 0.060$
    & $1.990 \pm 0.159$
    & \checkmark \\
\midrule
DCO-Warmstart
    & $\mathbf{0.805 \pm 0.066}$
    & $\mathbf{1.914 \pm 0.192}$
    & \checkmark \\
DirectTune
    & $0.805 \pm 0.072$
    & $1.914 \pm 0.163$
    & $\times$ \\
\bottomrule
\end{tabular}%
}
\end{table}

\begin{figure}[h]
    \centering
    \includegraphics[width=0.5\linewidth]
        {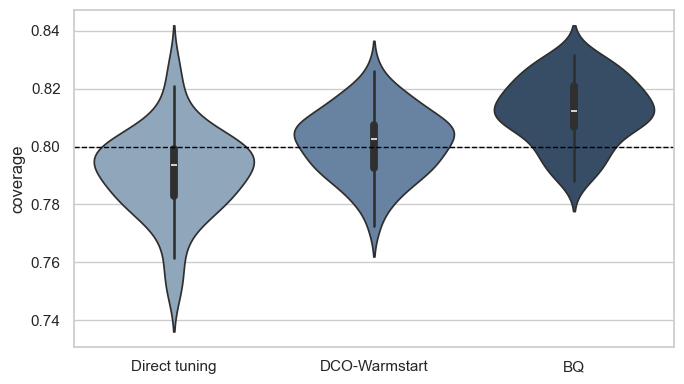}
    \caption{Coverage (a) and interval width (b) on 
    the Diabetes dataset over 50 random splits 
    ($1-\alpha=0.8$). The dashed line marks the 
    target level.}
    \label{fig:regression_results}
\end{figure}

\subsection{Classification: ImageNet-A}

\paragraph{Setup.}
We evaluate DCO-Warmstart on a filtered subset of ImageNet-A 
(198 classes) with target coverage $1-\alpha=0.8$, 
using a pretrained ResNet-50 backbone with a 
two-layer MC-dropout classification head ($T=20$ 
stochastic forward passes). Data are partitioned 
into approximately 2000 samples each for training, 
calibration, and testing, and 1000 for tuning, 
across 50 stratified random seeds. DCO-Warmstart searches over 
score type (\texttt{posterior\_nll} or 
\texttt{aoi\_nll}), dropout rate 
$\in\{0.05,0.1,0.2,0.3\}$, and hidden width 
$\in\{(512,256),(256,128)\}$, giving $|\Phi|=16$ 
candidates. Full specifications are in 
Appendix~\ref{app:imageneta}.

\paragraph{Results.}

\begin{table}[h]
\centering
\caption{ImageNet-A classification over 50 random 
splits (target $1-\alpha=0.8$). 
Paired Wilcoxon $p$-values (DCO-Warmstart vs.\ BQ, $n=50$): 
coverage $p=9.34\times10^{-9}$, average set size 
$p=9.77\times10^{-9}$, P95 set size 
$p=1.69\times10^{-9}$.}
\label{tab:imagenet_main}
\resizebox{0.65\linewidth}{!}{%
\begin{tabular}{lccc}
\toprule
Method & Coverage & Avg.\ Set Size & P95 Set Size \\
\midrule
BQ
    & $0.812\pm0.011$
    & $26.52\pm1.34$
    & $58.95\pm3.04$ \\
DCO-Warmstart    
    & $\mathbf{0.801\pm0.011}$ 
    & $\mathbf{25.26\pm1.60}$ 
    & $\mathbf{53.73\pm3.11}$ \\
DirectTune 
    & $0.791\pm0.015$ 
    & $23.83\pm1.80$ 
    & $50.65\pm4.02$ \\
\bottomrule
\end{tabular}%
}
\end{table}

DCO-Warmstart tracks the nominal coverage level closely 
($0.801$ vs target $0.8$), reduces the average 
prediction set from $26.52$ to $25.26$, and narrows 
the P95 tail from $58.95$ to $53.73$, with all 
three differences statistically significant 
(Table~\ref{tab:imagenet_main}). DirectTune 
achieves smaller sets but falls below the nominal 
coverage level in this classification experiment 
($0.791$ vs.\ target $0.8$), consistent with 
Remark~\ref{rem:version2-app}. Per-seed selection 
frequencies are in 
Table~\ref{tab:imagenet_selection_freq}.

\begin{figure}[h]
    \centering
    \includegraphics[width=0.85\linewidth]
    {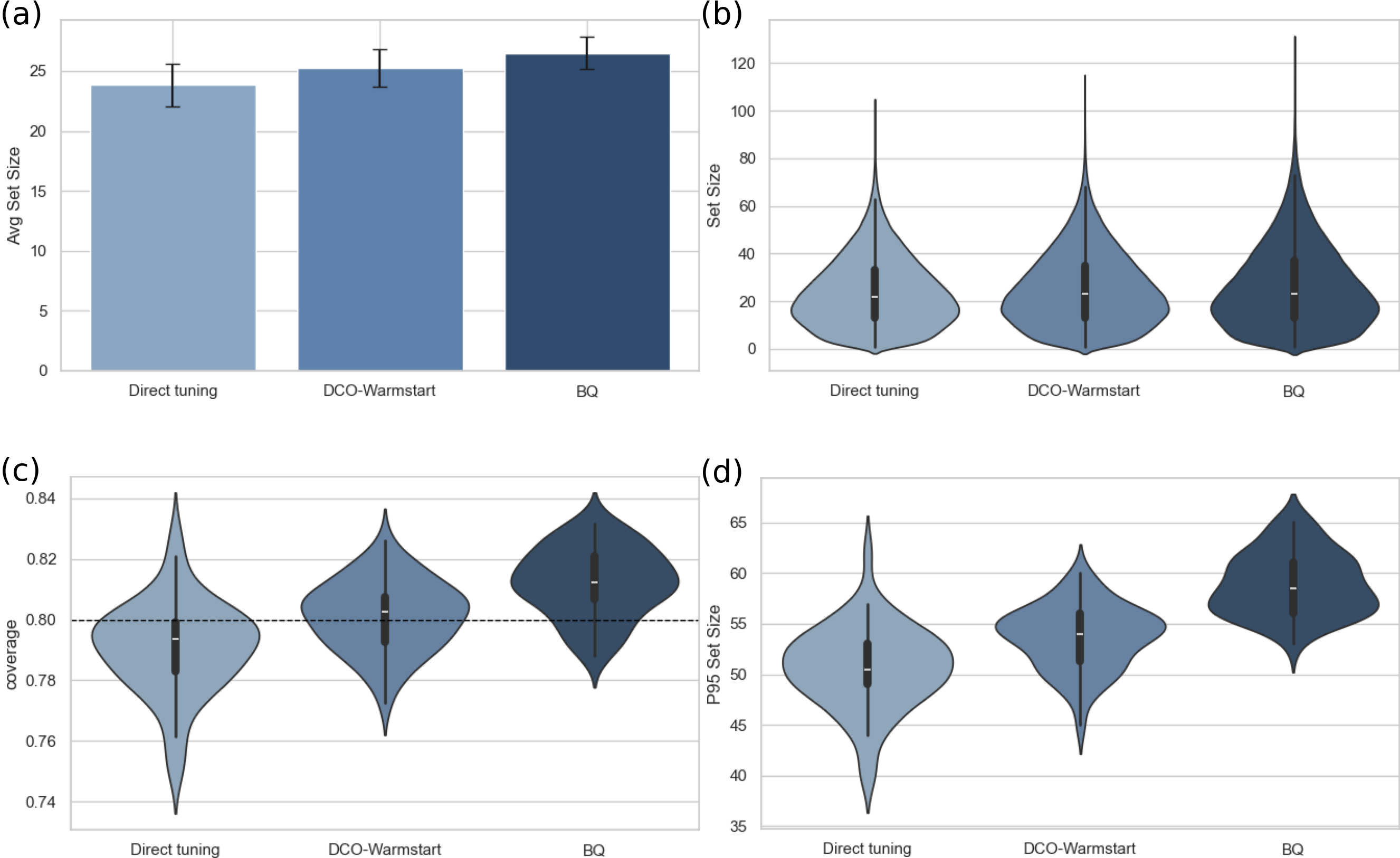}
    \caption{Classification results on ImageNet-A 
    over 50 random splits. (a)~Mean prediction set 
    size. (b)~Set size distributions. 
    (c)~Empirical coverage; dashed line marks 
    $1-\alpha=0.8$. (d)~P95 set size distributions.}
    \label{fig:imageneta_classification}
\end{figure}

\paragraph{Single-run illustration.}
Table~\ref{tab:dco_single_run} traces the threshold 
progression for a representative seed. The tuning 
threshold $\lambda_{\mathrm{tune}}=6.18$ falls below the 
nominal coverage target when applied directly for this 
seed; recalibration raises it to 
$q_{\mathrm{cal}}=6.42$, restoring coverage to 
$0.797$. BQ selects the more conservative 
$\lambda_{\mathrm{BQ}}=6.510$, recovering higher 
coverage $(0.808)$ at the cost of $4\%$ larger prediction 
sets. The achieved BQ feasibility probability at this 
threshold is 
$\widehat{p}_{\mathrm{BQ}}
= \widehat{\mathbb{P}}(L^+ \leq \alpha)
= 0.954 > 1-\delta = 0.95$, confirming that BQ satisfies 
its risk-control constraint with a small margin above the 
nominal confidence level.
\begin{table}[h]
\small
\centering
\caption{Single-seed threshold progression on ImageNet-A. 
Conformal recalibration raises the tuning threshold from 
$\lambda_{\mathrm{tune}}$ to $q_{\mathrm{cal}}$, restoring the 
split-conformal coverage guarantee. BQ selects a more conservative 
threshold $\lambda_{\mathrm{BQ}}$. Here 
$\widehat{p}_{\mathrm{BQ}} =
\widehat{\mathbb{P}}(L^+ \leq \alpha)$ is the achieved posterior feasibility 
probability estimated by Dirichlet sampling. It can exceed the nominal confidence 
level $0.95$ because BQ selects the smallest feasible threshold rather than 
enforcing equality.}
\label{tab:dco_single_run}
\resizebox{0.55\linewidth}{!}{%
\begin{tabular}{lcccc}
\toprule
& $\lambda_{\mathrm{tune}}$ 
& $q_{\mathrm{cal}}$ 
& $\lambda_{\mathrm{BQ}}$ 
& $\widehat{p}_{\mathrm{BQ}}$ \\
\midrule
Threshold & $6.180$ & $6.419$ & $6.510$ & $0.954$ \\
\midrule
Method & Coverage & Avg.\ Size & P95 & \\
\midrule
DCO-Warmstart & $0.797$ & $24.56$ & $54.0$ & \\
BQ          & $0.808$ & $25.56$ & $56.0$ & \\
\bottomrule
\end{tabular}%
}
\end{table}
\paragraph{Discussion.}
The two experiments tell a consistent story. Where 
the candidate space is large and the tuning signal 
strong, as in classification, DCO-Warmstart's efficiency gains 
over BQ are pronounced; where it is smaller and 
noisier, as in regression, the advantage is more 
modest but persists. In both cases, the core insight 
holds: decoupling optimisation from calibration is 
not only theoretically valid but empirically 
beneficial, freeing the tuning stage to select more 
efficient prediction rules while retaining the finite-sample marginal coverage guarantee of split conformal calibration.

\paragraph{Additional ablations.}
To further isolate the effect of decoupling, we additionally report ablations over candidate search, split allocation, and target coverage level. These ablations support a practical split-allocation principle: for a fixed non-training budget, the tuning split should be large enough to stabilise candidate selection, while the remaining data should be allocated to calibration to reduce conformal-quantile variability.

\begin{table}[h!]
\centering
\caption{
Matched candidate-class and calibration-mechanism controls on ImageNet-A. Because BQ/CRC
does not natively include DCO-style structural search, we report fixed-structure baselines,
BQ/CRC recalibration after DCO-Warmstart-selected structure, and an exploratory matched-$\Phi$ BQ/CRC
extension. These comparisons should be read as diagnostic controls rather than as a claim that
BQ/CRC and DCO-Warmstart implement identical search procedures.
}
\label{tab:matched_candidate}
\resizebox{\textwidth}{!}{%
\begin{tabular}{lcccc}
\toprule
Method & Candidate class & Coverage & Avg.\ size/width & P95 size/width \\
\midrule
BQ/CRC, fixed $\phi$ 
    & fixed 
    & $0.8122 \pm 0.0112$ 
    & $26.5192 \pm 1.3438$ 
    & $58.9460 \pm 3.0425$ \\
DCO-Warmstart, fixed $\phi$ 
    & fixed 
    & $0.8025 \pm 0.0108$ 
    & $25.0450 \pm 1.3687$ 
    & $55.6630 \pm 3.1192$ \\
Exploratory matched-$\Phi$ BQ/CRC 
    & matched $\Phi$ 
    & $0.8196 \pm 0.0116$ 
    & $28.0215 \pm 1.8460$ 
    & $59.8660 \pm 3.6270$ \\
DCO-Warmstart, search over $\Phi$ 
    & matched $\Phi$ 
    & $0.8013 \pm 0.0114$ 
    & $25.2584 \pm 1.5990$ 
    & $53.7260 \pm 3.1137$ \\
DirectTune 
    & matched $\Phi$ 
    & $0.7908 \pm 0.0151$ 
    & $23.8286 \pm 1.8024$ 
    & $50.6460 \pm 4.0233$ \\
\bottomrule
\end{tabular}%
}
\end{table}

\begin{table}[h!]
\centering
\caption{
DCO-Warmstart split-ratio ablation. The total non-training budget
$|D_{\mathrm{tune}}|+|D_{\mathrm{cal}}|$ is fixed, while the allocation between tuning and
calibration varies. The current codebase provides the $20/80$, $33/67$, $50/50$, $67/33$, 
and $80/20$ settings; selected-candidate stability is reported as the fraction of seeds 
choosing the modal candidate.
}
\label{tab:split_ratio}
\resizebox{\textwidth}{!}{%
\begin{tabular}{lcccc}
\toprule
Tune/calibration ratio & Coverage & Avg.\ size/width & P95 size/width & Selected candidate stability \\
\midrule
20/80 
    & $0.8001 \pm 0.0121$ 
    & $25.1344 \pm 1.6044$ 
    & $52.7240 \pm 2.7653$ 
    & $32.0\%$ \\
33/67 
    & $0.8013 \pm 0.0114$ 
    & $25.2584 \pm 1.5990$ 
    & $53.7260 \pm 3.1137$ 
    & $48.0\%$ \\
50/50 
    & $0.7996 \pm 0.0122$ 
    & $24.9321 \pm 1.5405$ 
    & $52.9460 \pm 3.7823$ 
    & $46.0\%$ \\
67/33 
    & $0.8010 \pm 0.0141$ 
    & $25.0379 \pm 1.8596$ 
    & $53.6030 \pm 4.2108$ 
    & $52.0\%$ \\
80/20 
    & $0.8027 \pm 0.0168$ 
    & $25.2462 \pm 2.6829$ 
    & $53.8880 \pm 5.7259$ 
    & $48.0\%$ \\
\bottomrule
\end{tabular}%
}
\end{table}
Table~\ref{tab:matched_candidate} shows that the exploratory matched-$\Phi$
BQ/CRC control produces a larger average set size than the fixed-structure BQ/CRC
baseline ($28.02$ versus $26.52$). The increase mainly comes from data allocation:
part of the non-training budget is used for tuning, leaving less effective calibration
information for the high-probability risk constraint and leading to a more conservative threshold.
DCO-Warmstart avoids this by separating candidate ranking on $D_{\mathrm{tune}}$ from
final calibration on $D_{\mathrm{cal}}$, keeping its average set size stable at
$25.26$.
\paragraph{Split-allocation principle.}
DCO-Warmstart introduces an allocation trade-off absent from standard split CP. Given a
fixed non-training budget
$m = m_{\mathrm{tune}} + m_{\mathrm{cal}}$, the tuning split controls the
stability of structural selection, while the calibration split controls the
variability of the final conformal quantile. If the best candidate is well
separated, a small tuning split may suffice; if several candidates have similar
efficiency, more tuning data may be needed.

Table~\ref{tab:split_ratio} evaluates this trade-off on ImageNet-A. Coverage and
average set size remain stable across all five allocations, with coverage ranging
from $0.7996$ to $0.8027$ and average set size from $24.93$ to $25.26$.
Candidate stability rises from $32\%$ under the $20/80$ split to $48\%$ under
the $33/67$ split, then fluctuates between $46\%$ and $52\%$ for larger tuning
fractions. In contrast, the standard deviation of the P95 set size increases as
$D_{\mathrm{cal}}$ shrinks, from $2.77$ at $20/80$ to $5.73$ at $80/20$.
These results suggest a simple rule: increase $m_{\mathrm{tune}}$ until
candidate selection stabilises, then allocate the remaining data to
$D_{\mathrm{cal}}$. On ImageNet-A, the $33/67$ split provides a good balance,
achieving near-maximal candidate stability while keeping the P95 variability
close to its minimum.
\begin{table}[h!]
\centering
\caption{
Performance across target coverage levels. Reporting multiple $\alpha$ values tests
whether DCO-Warmstart remains reliable beyond the single $1-\alpha=0.8$ setting.
}
\label{tab:alpha_levels}
\resizebox{\textwidth}{!}{%
\begin{tabular}{llccc}
\toprule
Dataset & Target coverage & Method & Empirical coverage & Avg.\ size/width \\
\midrule
ImageNet-A & 0.80 & BQ/CRC & $0.8122 \pm 0.0112$ & $26.5192 \pm 1.3438$ \\
ImageNet-A & 0.80 & DCO-Warmstart & $0.8012 \pm 0.0109$ & $25.2442 \pm 1.5561$ \\
ImageNet-A & 0.90 & BQ/CRC & $0.9093 \pm 0.0075$ & $49.3689 \pm 2.1015$ \\
ImageNet-A & 0.90 & DCO-Warmstart & $0.9014 \pm 0.0080$ & $46.8652 \pm 2.1532$ \\
ImageNet-A & 0.95 & BQ/CRC & $0.9580 \pm 0.0053$ & $74.4952 \pm 3.1281$ \\
ImageNet-A & 0.95 & DCO-Warmstart & $0.9515 \pm 0.0067$ & $71.2709 \pm 3.6914$ \\
CIFAR-100 & 0.80 & BQ/CRC & $0.8110 \pm 0.0106$ & $2.7219 \pm 0.1434$ \\
CIFAR-100 & 0.80 & DCO-Warmstart & $0.7994 \pm 0.0125$ & $2.5270 \pm 0.1634$ \\
CIFAR-100 & 0.90 & BQ/CRC & $0.9071 \pm 0.0087$ & $5.8332 \pm 0.4278$ \\
CIFAR-100 & 0.90 & DCO-Warmstart & $0.8978 \pm 0.0101$ & $5.3802 \pm 0.4791$ \\
CIFAR-100 & 0.95 & BQ/CRC & $0.9547 \pm 0.0060$ & $10.7320 \pm 0.7601$ \\
CIFAR-100 & 0.95 & DCO-Warmstart & $0.9488 \pm 0.0085$ & $10.0229 \pm 1.2058$ \\
\bottomrule
\end{tabular}%
}
\end{table}

\begin{table}[H]
\centering
\caption{
Additional regression benchmarks across target coverage levels on California Housing and
Concrete. Both BQ/CRC and DCO-Warmstart are reported using completed reruns under the matched-budget
protocol.
}
\label{tab:regression_alpha_levels}
\resizebox{\textwidth}{!}{%
\begin{tabular}{llccc}
\toprule
Dataset & Target coverage & Method & Empirical coverage & Avg.\ width \\
\midrule
California Housing & 0.80 & BQ/CRC & $0.8080 \pm 0.0273$ & $1.3507 \pm 0.0472$ \\
California Housing & 0.80 & DCO-Warmstart & $0.7902 \pm 0.0287$ & $1.2917 \pm 0.0558$ \\
California Housing & 0.90 & BQ/CRC & $0.9112 \pm 0.0185$ & $1.9377 \pm 0.0867$ \\
California Housing & 0.90 & DCO-Warmstart & $0.8996 \pm 0.0205$ & $1.8339 \pm 0.0932$ \\
California Housing & 0.95 & BQ/CRC & $0.9586 \pm 0.0118$ & $2.7078 \pm 0.1572$ \\
California Housing & 0.95 & DCO-Warmstart & $0.9511 \pm 0.0144$ & $2.5302 \pm 0.1768$ \\
Concrete & 0.80 & BQ/CRC & $0.8226 \pm 0.0432$ & $1.7109 \pm 0.0818$ \\
Concrete & 0.80 & DCO-Warmstart & $0.7934 \pm 0.0490$ & $1.5821 \pm 0.1278$ \\
Concrete & 0.90 & BQ/CRC & $0.9195 \pm 0.0309$ & $2.2582 \pm 0.1000$ \\
Concrete & 0.90 & DCO-Warmstart & $0.8937 \pm 0.0316$ & $2.0915 \pm 0.1252$ \\
Concrete & 0.95 & BQ/CRC & $0.9640 \pm 0.0188$ & $2.7371 \pm 0.1268$ \\
Concrete & 0.95 & DCO-Warmstart & $0.9462 \pm 0.0252$ & $2.5274 \pm 0.1579$ \\
\bottomrule
\end{tabular}%
}
\end{table}

\section{Conclusion}
\label{sec:conclusion}

We studied whether optimisation and final conformal calibration must use the same
held-out data in Bayesian conformal optimisation pipelines. When the target is
finite-sample marginal conformal coverage, they need not. Decoupled Conformal
Optimisation (DCO) assigns these roles to separate splits. The tuning split
selects the score, model, prior, or threshold-search configuration. The
calibration split is then used only to compute the final conformal quantile.

This separation preserves the standard split-conformal logic. Once the tuned
structure is fixed independently of $D_{\mathrm{cal}}$, the calibration scores
and the test score remain exchangeable. Theorem~\ref{thm:dco-coverage} therefore
gives finite-sample marginal coverage for any candidate class $\Phi$, without a
confidence parameter and without a multiple-testing correction over the
candidate class. Proposition~\ref{prop:tuning_oracle} complements this result by
describing when the tuning split is large enough to select an efficient candidate
from a finite class. Proposition~\ref{prop:asymptotic} provides a large-sample
comparison: under consistency assumptions on the coupled risk bound, DCO-Warmstart and
CRC/BQ-style calibration converge to the same population threshold, although
their finite-sample guarantees remain distinct.

This distinction is important. DCO-Warmstart is not a replacement for CRC/BQ-style methods
when the scientific goal is high-probability risk control. Those methods provide
a different type of guarantee, controlled by a confidence level $1-\delta$. DCO-Warmstart
is aimed at the marginal coverage setting. In that setting, coupling optimisation
and calibration is sufficient but not necessary. An independent tuning split can
be used for efficiency-oriented search, while a fresh calibration split supplies
the conformal coverage guarantee.

The experiments support this view. Under matched protocols, DCO-Warmstart retains the
split-conformal marginal coverage guarantee and, in finite test evaluations,
tracks the nominal level closely while often producing smaller prediction sets or intervals than
coupled high-probability calibration baselines. On ImageNet-A, the average set
size decreases from $26.52$ to $25.26$, with the 95th-percentile set size
decreasing from $58.95$ to $53.73$, indicating improvement in the tail of the
set-size distribution. On the Diabetes regression benchmark, the average interval
width decreases from $2.098$ to $1.914$. DirectTune illustrates the cost of
omitting the final calibration step: it may achieve smaller sets at the cost of
under-coverage or increased coverage variability, since it lacks a finite-sample
conformal guarantee.

\subsection{Limitations and future work.}
DCO-Warmstart introduces a practical allocation problem. Data must be divided among
training, tuning, and calibration, and the best split ratio depends on the task.
The tuning split should be large enough to stabilise candidate selection. The
calibration split should remain large enough to reduce quantile variability. A
simple diagnostic is to monitor whether the selected
$\hat{\phi}_{\mathrm{tune}}$ changes under repeated random splits of the tuning
data. Stability of this selection suggests that the tuning budget is sufficient; the remaining non-training data can then be allocated to calibration.

Several theoretical questions remain open. The current tuning-stage oracle result
covers finite candidate classes. Extensions to adaptive split allocation,
continuous hyperparameter spaces, and end-to-end efficiency after recalibration
would make the theory more complete. The asymptotic comparison in
Proposition~\ref{prop:asymptotic} also relies on consistency of the coupled risk
bound. Understanding when this condition holds for specific CRC/BQ constructions,
and how large the finite-sample conservativeness gap remains for fixed
$\delta$, are useful directions for future work.

\section*{Acknowledgement}
This work was supported by the Engineering and Physical Sciences Research Council (EPSRC) under Grant No. EP/Y030826/1. 
\newpage
\bibliography{reference}
\clearpage
\appendix
\appendix
\section{Split Conformal Prediction}
\label{app:splitcp}

Algorithm~\ref{alg:splitcp} summarises the standard 
split conformal prediction pipeline. The procedure 
requires no assumptions beyond exchangeability of the 
data and a bounded non-conformity score, and the 
coverage guarantee follows directly from the order 
statistics of the calibration scores.

\begin{algorithm}[h]
\caption{Split Conformal Prediction}
\label{alg:splitcp}
\begin{algorithmic}[1]
\REQUIRE Training data $D_{\text{train}}$, 
         calibration data $D_{\text{cal}}$,
         miscoverage level $\alpha$
\STATE Train model $\hat{f}$ on $D_{\text{train}}$
\STATE Compute non-conformity scores 
       $S_i = S(X_i, Y_i)$ for each 
       $(X_i, Y_i) \in D_{\text{cal}}$
\STATE Sort $S_1,\ldots,S_m$ into 
       $S_{(1)} \le \cdots \le S_{(m)}$
\STATE Set $k_\alpha=\left\lceil (m+1)(1-\alpha)\right\rceil$
\STATE Set $\hat{q}_{1-\alpha}=S_{(k_\alpha)}$ if 
       $k_\alpha\le m$, and 
       $\hat{q}_{1-\alpha}=+\infty$ otherwise
\ENSURE $C(x) = \{y : S(x, y) \leq \hat{q}_{1-\alpha}\}$,
        which satisfies 
        $\mathbb{P}(Y_{n+1} \in C(X_{n+1})) \geq 1 - \alpha$
\end{algorithmic}
\end{algorithm}

The marginal coverage guarantee follows from a standard 
exchangeability argument. Since the calibration points 
and the test point $(X_{n+1}, Y_{n+1})$ are exchangeable, 
the test non-conformity score $S_{n+1}$ is equally 
likely to fall at any rank among 
$S_1, \ldots, S_{m_{\text{cal}}}, S_{n+1}$. Therefore,
\begin{equation}
    \mathbb{P}(Y_{n+1} \in C(X_{n+1})) 
    = \mathbb{P}(S_{n+1} \leq \hat{q}_{1-\alpha}) 
    \geq 1 - \alpha.
\end{equation}
Crucially, $\hat{q}_{1-\alpha}$ is an order statistic 
of $D_{\text{cal}}$ and involves no optimisation; the 
efficiency of $C(x)$ therefore depends entirely on the 
expressiveness of the non-conformity score $S(x,y)$.
\section{Proofs and Technical Details}
\label{app:proofs}

This appendix provides complete proofs for all theoretical 
results stated in Section~\ref{Sec: Theory}, together with 
supporting lemmas and extensions.

\subsection{Proof of Theorem~\ref{thm:dco-coverage} 
(Marginal Coverage of DCO-Warmstart)}

\begin{proof}
Let $\hat{\phi}_{\text{tune}}$ be selected using 
$D_{\text{train}}$ and $D_{\text{tune}}$ only. Let 
$D_{\mathrm{cal}}=\{(X_i,Y_i)\}_{i=1}^{m}$, and define 
the calibration scores
\begin{equation}
    S_i = S_{\hat{\phi}_{\text{tune}}}(X_i,Y_i),
    \qquad i=1,\ldots,m.
\end{equation}
Define the test score as
\begin{equation}
    S_{m+1}
    =
    S_{\hat{\phi}_{\text{tune}}}(X_{m+1},Y_{m+1}).
\end{equation}
Let
\begin{equation}
    S_{(1)}\le\cdots\le S_{(m)}
\end{equation}
denote the sorted calibration scores, and define
\begin{equation}
    k_\alpha
    =
    \left\lceil(m+1)(1-\alpha)\right\rceil .
\end{equation}
The split-conformal calibration threshold is
\begin{equation}
    \hat{q}_{1-\alpha}
    =
    \begin{cases}
        S_{(k_\alpha)}, & k_\alpha \le m,\\
        +\infty,        & k_\alpha = m+1.
    \end{cases}
\end{equation}

Since $\hat{\phi}_{\text{tune}}$ depends only on 
$D_{\text{train}}$ and $D_{\text{tune}}$, and is not 
adapted to $D_{\mathrm{cal}}$, the score function 
$S_{\hat{\phi}_{\text{tune}}}(\cdot,\cdot)$ is fixed 
with respect to the calibration data. Conditional on 
$D_{\text{train}}$, $D_{\text{tune}}$, and 
$\hat{\phi}_{\text{tune}}$, the calibration points and 
the test point,
\begin{equation}
    (X_1,Y_1),\ldots,(X_m,Y_m),(X_{m+1},Y_{m+1}),
\end{equation}
are exchangeable. Therefore, the scores
\begin{equation}
    S_1,\ldots,S_m,S_{m+1}
\end{equation}
are exchangeable conditional on 
$D_{\text{train}}$, $D_{\text{tune}}$, and 
$\hat{\phi}_{\text{tune}}$.

By the standard split-conformal coverage argument 
\cite{vovk2005algorithmic,shafer2008tutorial}, 
exchangeability of the scores implies
\begin{equation}
    \mathbb{P}\!\left(
        S_{m+1}\le\hat{q}_{1-\alpha}
        \;\middle|\;
        D_{\text{train}},D_{\text{tune}},
        \hat{\phi}_{\text{tune}}
    \right)
    \ge 1-\alpha.
\end{equation}
Since
\begin{equation}
    Y_{m+1}\in 
    C_{\hat{\phi}_{\text{tune}},\hat{q}_{1-\alpha}}(X_{m+1})
    \quad
    \Longleftrightarrow
    \quad
    S_{m+1}\le \hat{q}_{1-\alpha},
\end{equation}
taking expectations over $D_{\text{train}}$, 
$D_{\text{tune}}$, and $\hat{\phi}_{\text{tune}}$ gives
\begin{equation}
    \mathbb{P}\!\left(
        Y_{m+1}\in
        C_{\hat{\phi}_{\text{tune}},\hat{q}_{1-\alpha}}
        (X_{m+1})
    \right)
    \ge 1-\alpha.
    \qedhere
\end{equation}
\end{proof}

\begin{remark}[Score selection and multiple testing]
\label{rem:multiple-testing-app}
The argument above holds for any candidate class $\Phi$, 
finite or infinite. The key requirement is that calibration 
is applied to a single fixed structure 
$\hat{\phi}_{\text{tune}}$ selected without using 
$D_{\mathrm{cal}}$. Unlike procedures that select among 
multiple thresholds using the calibration data, DCO-Warmstart does 
not require a union bound or family-wise error correction 
over $\Phi$ for its final conformal coverage guarantee.
\end{remark}

\begin{remark}[Failure of DirectTune]
\label{rem:version2-app}
DirectTune selects
\begin{equation}
    \hat{\lambda}_{\text{tune}}
    =
    \operatorname*{arg\,min}_{\lambda}
    \widehat{\mathcal{S}}_{\text{tune}}(\lambda)
    \quad\text{s.t.}\quad
    \widehat{R}_{\text{tune}}(\lambda)\le\alpha,
\end{equation}
and deploys this threshold directly without further 
calibration. This procedure is not conformally certified. 
For any fixed $\lambda$, the empirical risk 
$\widehat{R}_{\text{tune}}(\lambda)$ estimates the 
corresponding population risk $R(\lambda)$. However, the 
selected threshold $\hat{\lambda}_{\text{tune}}$ is itself 
a function of $D_{\text{tune}}$. Thus, the same data are 
used both to choose the threshold and to certify its 
empirical feasibility. This selection effect can introduce 
optimistic bias, so the empirical feasibility constraint 
does not imply finite-sample marginal coverage for future 
test points. In finite samples, this can manifest as 
undercoverage or increased coverage variability.
\end{remark}

\subsection{Proof of Proposition~\ref{prop:tuning_oracle} 
(Tuning oracle inequality for a finite search class)}

\begin{proof}
Let $\mathcal{A}$ denote a finite search class, where each 
$a\in\mathcal{A}$ represents a candidate procedure, such as 
a pair $a=(\phi,\lambda)$. Define its population 
miscoverage and population size by
\begin{equation}
    R(a)
    =
    \mathbb{P}\!\left\{Y\notin C_a(X)\right\},
    \qquad
    \mathcal{S}(a)
    =
    \mathbb{E}\!\left[s(C_a(X))\right],
\end{equation}
and their empirical counterparts on $D_{\mathrm{tune}}$ by
\begin{equation}
    \widehat{R}_{\mathrm{tune}}(a)
    =
    \frac{1}{m_{\mathrm{tune}}}
    \sum_{i=1}^{m_{\mathrm{tune}}}
    \mathbf{1}\{Y_i\notin C_a(X_i)\},
\end{equation}
and
\begin{equation}
    \widehat{\mathcal{S}}_{\mathrm{tune}}(a)
    =
    \frac{1}{m_{\mathrm{tune}}}
    \sum_{i=1}^{m_{\mathrm{tune}}}
    s(C_a(X_i)).
\end{equation}
Assume that the size functional is bounded as
\begin{equation}
    0\le s(C_a(X))\le B
    \qquad
    \text{for all } a\in\mathcal{A}.
\end{equation}

For each fixed $a\in\mathcal{A}$, the random variable 
$\mathbf{1}\{Y\notin C_a(X)\}$ is bounded in $[0,1]$. 
Hoeffding's inequality gives
\begin{equation}
    \mathbb{P}\!\left(
        \left|
            \widehat{R}_{\mathrm{tune}}(a)-R(a)
        \right|
        >
        \varepsilon_R
    \right)
    \le
    2\exp(-2m_{\mathrm{tune}}\varepsilon_R^2).
\end{equation}
Applying a union bound over $a\in\mathcal{A}$ yields
\begin{equation}
    \mathbb{P}\!\left(
        \sup_{a\in\mathcal{A}}
        \left|
            \widehat{R}_{\mathrm{tune}}(a)-R(a)
        \right|
        >
        \varepsilon_R
    \right)
    \le
    2|\mathcal{A}|
    \exp(-2m_{\mathrm{tune}}\varepsilon_R^2).
\end{equation}

Similarly, since $s(C_a(X))/B\in[0,1]$, Hoeffding's 
inequality gives
\begin{equation}
    \mathbb{P}\!\left(
        \left|
            \widehat{\mathcal{S}}_{\mathrm{tune}}(a)
            -
            \mathcal{S}(a)
        \right|
        >
        \varepsilon_S
    \right)
    \le
    2\exp\!\left(
        -\frac{2m_{\mathrm{tune}}\varepsilon_S^2}{B^2}
    \right).
\end{equation}
A second union bound over $a\in\mathcal{A}$ yields
\begin{equation}
    \mathbb{P}\!\left(
        \sup_{a\in\mathcal{A}}
        \left|
            \widehat{\mathcal{S}}_{\mathrm{tune}}(a)
            -
            \mathcal{S}(a)
        \right|
        >
        \varepsilon_S
    \right)
    \le
    2|\mathcal{A}|
    \exp\!\left(
        -\frac{2m_{\mathrm{tune}}\varepsilon_S^2}{B^2}
    \right).
\end{equation}

Therefore, if
\begin{equation}
    m_{\mathrm{tune}}
    \ge
    \max\left\{
        \frac{\log(4|\mathcal{A}|/\eta)}
        {2\varepsilon_R^2},
        \frac{B^2\log(4|\mathcal{A}|/\eta)}
        {2\varepsilon_S^2}
    \right\},
    \label{eq:tuning-sample-complexity-app}
\end{equation}
then with probability at least $1-\eta$, the following 
two uniform deviation bounds hold simultaneously:
\begin{equation}
    \sup_{a\in\mathcal{A}}
    \left|
        \widehat{R}_{\mathrm{tune}}(a)-R(a)
    \right|
    \le
    \varepsilon_R,
\end{equation}
and
\begin{equation}
    \sup_{a\in\mathcal{A}}
    \left|
        \widehat{\mathcal{S}}_{\mathrm{tune}}(a)
        -
        \mathcal{S}(a)
    \right|
    \le
    \varepsilon_S.
\end{equation}

On this event, let $\widehat{a}$ be an empirical minimizer 
of average size subject to the empirical miscoverage 
constraint:
\begin{equation}
    \widehat{a}
    \in
    \operatorname*{arg\,min}_{a\in\mathcal{A}}
    \widehat{\mathcal{S}}_{\mathrm{tune}}(a)
    \quad
    \text{s.t.}
    \quad
    \widehat{R}_{\mathrm{tune}}(a)\le\alpha.
\end{equation}
Then its population miscoverage satisfies
\begin{equation}
    R(\widehat{a})
    \le
    \widehat{R}_{\mathrm{tune}}(\widehat{a})
    +
    \varepsilon_R
    \le
    \alpha+\varepsilon_R.
\end{equation}
Moreover, any candidate $a\in\mathcal{A}$ satisfying
\begin{equation}
    R(a)\le \alpha-\varepsilon_R
\end{equation}
is empirically feasible, because
\begin{equation}
    \widehat{R}_{\mathrm{tune}}(a)
    \le
    R(a)+\varepsilon_R
    \le
    \alpha.
\end{equation}
Hence, for every such candidate $a$, the empirical 
minimizer satisfies
\begin{equation}
    \widehat{\mathcal{S}}_{\mathrm{tune}}(\widehat{a})
    \le
    \widehat{\mathcal{S}}_{\mathrm{tune}}(a).
\end{equation}
Using the uniform size deviation bound on both sides gives
\begin{equation}
    \mathcal{S}(\widehat{a})
    \le
    \widehat{\mathcal{S}}_{\mathrm{tune}}(\widehat{a})
    +
    \varepsilon_S
    \le
    \widehat{\mathcal{S}}_{\mathrm{tune}}(a)
    +
    \varepsilon_S
    \le
    \mathcal{S}(a)+2\varepsilon_S.
\end{equation}
Taking the infimum over all candidates satisfying 
$R(a)\le\alpha-\varepsilon_R$ yields
\begin{equation}
    \mathcal{S}(\widehat{a})
    \le
    \inf_{a\in\mathcal{A}:\,R(a)\le\alpha-\varepsilon_R}
    \mathcal{S}(a)
    +
    2\varepsilon_S.
\end{equation}
Thus, with probability at least $1-\eta$, the selected 
candidate has population miscoverage at most 
$\alpha+\varepsilon_R$ and size within $2\varepsilon_S$ 
of the best candidate whose population miscoverage is at 
most $\alpha-\varepsilon_R$.

If one instead uses the tightened empirical constraint
\begin{equation}
    \widehat{R}_{\mathrm{tune}}(a)\le\alpha-\varepsilon_R,
\end{equation}
then the same argument gives the stronger feasibility 
statement
\begin{equation}
    R(\widehat{a})\le\alpha,
\end{equation}
with the oracle comparison taken over candidates satisfying
\begin{equation}
    R(a)\le\alpha-2\varepsilon_R.
\end{equation}
This proves the proposition.
\end{proof}

\paragraph{Possible extension.}
An extension to infinite or continuous candidate classes 
would require additional uniform-convergence machinery, 
such as covering-number or empirical-process arguments, 
together with explicit regularity assumptions on the 
search class. We leave such extensions to future work 
and do not claim them in the present paper.
\subsection{Calibration Quantile Accuracy}

\begin{lemma}[Calibration quantile accuracy]
\label{lem:cal-accuracy}
Fix a tuned structure $\hat{\phi}_{\text{tune}}$ and let 
$F_{\hat{\phi}_{\text{tune}}}$ denote the distribution 
function of 
$S_{\hat{\phi}_{\text{tune}}}(X,Y)$. Define the population 
quantile
\begin{equation}
    q_{\hat{\phi}_{\text{tune}}}^{\star}
    =
    \inf\left\{
        q :
        F_{\hat{\phi}_{\text{tune}}}(q) \ge 1-\alpha
    \right\}.
\end{equation}
Let $m=m_{\mathrm{cal}}$, let 
$S_{(1)}\le\cdots\le S_{(m)}$ denote the sorted calibration 
scores, and define
\begin{equation}
    k_\alpha
    =
    \left\lceil (m+1)(1-\alpha)\right\rceil .
\end{equation}
Assume that $k_\alpha\le m$, and let
\begin{equation}
    \hat{q}_{1-\alpha}
    =
    S_{(k_\alpha)}.
\end{equation}
Suppose that $F_{\hat{\phi}_{\text{tune}}}$ is continuous 
and has density bounded below by $c>0$ in a neighbourhood 
of $q_{\hat{\phi}_{\text{tune}}}^{\star}$. Then, for any 
$t>0$ such that 
$q_{\hat{\phi}_{\text{tune}}}^{\star}\pm t$ remain in this 
neighbourhood,
\begin{equation}
    \mathbb{P}\!\left(
        \left|
            \hat{q}_{1-\alpha}
            -
            q_{\hat{\phi}_{\text{tune}}}^{\star}
        \right|
        > t
    \right)
    \le
    2\exp\!\left\{
        -2m
        \left(
            c t - \frac{2}{m}
        \right)_{+}^{2}
    \right\},
    \label{eq:cal-quantile-accuracy}
\end{equation}
where $(u)_{+}=\max\{u,0\}$. In particular, if 
$ct>2/m$, the right-hand side decays exponentially in 
$m t^2$.
\end{lemma}

\begin{proof}
Let
\begin{equation}
    p = 1-\alpha.
\end{equation}
Let $\widehat{F}_{m}$ be the empirical distribution 
function of the $m$ calibration scores. By the 
Dvoretzky--Kiefer--Wolfowitz inequality,
\begin{equation}
    \mathbb{P}\!\left(
        \sup_{q}
        \left|
            \widehat{F}_{m}(q)
            -
            F_{\hat{\phi}_{\text{tune}}}(q)
        \right|
        > \epsilon
    \right)
    \le
    2\exp(-2m\epsilon^2).
    \label{eq:dkw-calibration}
\end{equation}
Because
\begin{equation}
    k_\alpha
    =
    \left\lceil (m+1)p\right\rceil ,
\end{equation}
we have
\begin{equation}
    \frac{k_\alpha}{m}
    \le
    p + \frac{2}{m}.
    \label{eq:rank-upper}
\end{equation}
Moreover, since $k_\alpha\ge (m+1)p$, we also have
\begin{equation}
    \frac{k_\alpha}{m}
    \ge
    p.
    \label{eq:rank-lower}
\end{equation}

By the lower density assumption, for 
$q_{\hat{\phi}_{\text{tune}}}^{\star}+t$ in the stated 
neighbourhood,
\begin{equation}
    F_{\hat{\phi}_{\text{tune}}}
    \!\left(
        q_{\hat{\phi}_{\text{tune}}}^{\star}+t
    \right)
    \ge
    p + ct.
    \label{eq:cdf-right}
\end{equation}
If 
$\hat{q}_{1-\alpha}>q_{\hat{\phi}_{\text{tune}}}^{\star}+t$, 
then fewer than $k_\alpha$ calibration scores are less than 
or equal to 
$q_{\hat{\phi}_{\text{tune}}}^{\star}+t$, and hence
\begin{equation}
    \widehat{F}_{m}
    \!\left(
        q_{\hat{\phi}_{\text{tune}}}^{\star}+t
    \right)
    <
    \frac{k_\alpha}{m}.
\end{equation}
Combining this with \eqref{eq:rank-upper} and 
\eqref{eq:cdf-right} gives
\begin{equation}
    \sup_q
    \left|
        \widehat{F}_{m}(q)
        -
        F_{\hat{\phi}_{\text{tune}}}(q)
    \right|
    >
    ct-\frac{2}{m}.
    \label{eq:right-tail-quantile}
\end{equation}

Similarly, by the lower density assumption, for 
$q_{\hat{\phi}_{\text{tune}}}^{\star}-t$ in the stated 
neighbourhood,
\begin{equation}
    F_{\hat{\phi}_{\text{tune}}}
    \!\left(
        q_{\hat{\phi}_{\text{tune}}}^{\star}-t
    \right)
    \le
    p - ct.
    \label{eq:cdf-left}
\end{equation}
If 
$\hat{q}_{1-\alpha}<q_{\hat{\phi}_{\text{tune}}}^{\star}-t$, 
then at least $k_\alpha$ calibration scores are less than 
or equal to 
$q_{\hat{\phi}_{\text{tune}}}^{\star}-t$, and hence
\begin{equation}
    \widehat{F}_{m}
    \!\left(
        q_{\hat{\phi}_{\text{tune}}}^{\star}-t
    \right)
    \ge
    \frac{k_\alpha}{m}.
\end{equation}
Using \eqref{eq:rank-lower} and \eqref{eq:cdf-left}, this 
implies
\begin{equation}
    \sup_q
    \left|
        \widehat{F}_{m}(q)
        -
        F_{\hat{\phi}_{\text{tune}}}(q)
    \right|
    >
    ct.
    \label{eq:left-tail-quantile}
\end{equation}

Combining \eqref{eq:right-tail-quantile} and 
\eqref{eq:left-tail-quantile}, we obtain
\begin{equation}
    \left\{
        \left|
            \hat{q}_{1-\alpha}
            -
            q_{\hat{\phi}_{\text{tune}}}^{\star}
        \right|
        > t
    \right\}
    \subseteq
    \left\{
        \sup_q
        \left|
            \widehat{F}_{m}(q)
            -
            F_{\hat{\phi}_{\text{tune}}}(q)
        \right|
        >
        \left(
            ct-\frac{2}{m}
        \right)_{+}
    \right\}.
\end{equation}
Applying \eqref{eq:dkw-calibration} with 
\begin{equation}
    \epsilon
    =
    \left(
        ct-\frac{2}{m}
    \right)_{+}
\end{equation}
gives \eqref{eq:cal-quantile-accuracy}.
\end{proof}

\subsection{Proof of Proposition~\ref{prop:asymptotic}
(Asymptotic Agreement under Uniform Risk-Bound 
Consistency)}

We restate the regularity conditions for completeness.

\begin{assumption}
\label{ass:asymptotic}
\begin{enumerate}[label=(A\arabic*)]
    \item The function $R(\lambda)$ is continuous and 
    strictly decreasing in a neighbourhood of
    \begin{equation}
        \lambda^\star
        =
        \inf\left\{
            \lambda :
            R(\lambda) \le \alpha
        \right\}.
    \end{equation}
    Moreover, $\lambda^\star$ is an interior point of 
    the search domain $\Lambda$.

    \item The split-conformal DCO-Warmstart threshold satisfies
    \begin{equation}
        \hat{\lambda}_{\mathrm{DCO}}
        \xrightarrow{p}
        \lambda^\star.
    \end{equation}

    \item The coupled CRC/BQ threshold can be written as
    \begin{equation}
        \hat{\lambda}_{\mathrm{CRC}}
        =
        \inf\!\left\{
            \lambda :
            \widehat{R}_m(\lambda)
            +
            b_m(\lambda,\delta_m)
            \le \alpha
        \right\},
    \end{equation}
    where $b_m(\lambda,\delta_m)\ge 0$ and the empirical 
    risk and excess margin satisfy, respectively,
    \begin{equation}
        \sup_{\lambda\in\Lambda}
        \left|
            \widehat{R}_m(\lambda) - R(\lambda)
        \right|
        \xrightarrow{p} 0,
    \end{equation}
    and
    \begin{equation}
        \sup_{\lambda\in\Lambda}
        b_m(\lambda,\delta_m)
        \xrightarrow{p} 0.
    \end{equation}
\end{enumerate}
\end{assumption}

\begin{proof}[Proof of Proposition~\ref{prop:asymptotic}]
By assumption~(A2),
\begin{equation}
    \hat{\lambda}_{\mathrm{DCO}}
    \xrightarrow{p}
    \lambda^\star.
\end{equation}
It remains to show that
\begin{equation}
    \hat{\lambda}_{\mathrm{CRC}}
    \xrightarrow{p}
    \lambda^\star.
\end{equation}

Fix any $\varepsilon>0$ small enough that 
$\lambda^\star-\varepsilon$ and 
$\lambda^\star+\varepsilon$ lie in the neighbourhood where 
$R$ is continuous and strictly decreasing. Since 
$\lambda^\star$ is the boundary of the population feasible 
set and $R$ is strictly decreasing near $\lambda^\star$, 
we have
\begin{equation}
    R(\lambda^\star-\varepsilon)
    >
    \alpha
    >
    R(\lambda^\star+\varepsilon).
\end{equation}
Define the positive margin
\begin{equation}
    \Delta_\varepsilon
    =
    \frac{1}{2}
    \min\left\{
        R(\lambda^\star-\varepsilon)-\alpha,\,
        \alpha-R(\lambda^\star+\varepsilon)
    \right\}
    >
    0.
\end{equation}

By assumption~(A3), with probability tending to one,
\begin{equation}
    \sup_{\lambda\in\Lambda}
    \left|
        \widehat{R}_m(\lambda)-R(\lambda)
    \right|
    \le
    \Delta_\varepsilon
\end{equation}
and
\begin{equation}
    \sup_{\lambda\in\Lambda}
    b_m(\lambda,\delta_m)
    \le
    \Delta_\varepsilon.
\end{equation}
On this event, since $b_m(\lambda,\delta_m)\ge 0$,
\begin{equation}
    \widehat{R}_m(\lambda^\star-\varepsilon)
    +
    b_m(\lambda^\star-\varepsilon,\delta_m)
    \ge
    R(\lambda^\star-\varepsilon)
    -
    \Delta_\varepsilon
    >
    \alpha.
\end{equation}
Thus $\lambda^\star-\varepsilon$ is not feasible. 
Similarly,
\begin{equation}
    \widehat{R}_m(\lambda^\star+\varepsilon)
    +
    b_m(\lambda^\star+\varepsilon,\delta_m)
    \le
    R(\lambda^\star+\varepsilon)
    +
    2\Delta_\varepsilon
    <
    \alpha.
\end{equation}
Thus $\lambda^\star+\varepsilon$ is feasible. Therefore, 
on an event whose probability tends to one,
\begin{equation}
    \lambda^\star-\varepsilon
    <
    \hat{\lambda}_{\mathrm{CRC}}
    \le
    \lambda^\star+\varepsilon.
\end{equation}
Equivalently,
\begin{equation}
    \left|
        \hat{\lambda}_{\mathrm{CRC}}
        -
        \lambda^\star
    \right|
    \le
    \varepsilon
\end{equation}
with probability tending to one. Hence,
\begin{equation}
    \hat{\lambda}_{\mathrm{CRC}}
    \xrightarrow{p}
    \lambda^\star.
\end{equation}
Combining this with assumption~(A2) yields
\begin{equation}
    \hat{\lambda}_{\mathrm{DCO}}
    -
    \hat{\lambda}_{\mathrm{CRC}}
    \xrightarrow{p}
    0.
\end{equation}
\end{proof}

\begin{remark}[Finite-sample distinction between DCO-Warmstart and CRC/BQ]
\label{rem:delta-app}
Proposition~\ref{prop:asymptotic} should not be 
interpreted as asserting that DCO-Warmstart and CRC/BQ provide 
the same finite-sample guarantee. They do not. DCO-Warmstart 
targets marginal conformal coverage, whereas CRC/BQ 
targets high-probability risk control. The proposition 
states only that, when the coupled risk bound 
consistently estimates the population risk boundary 
and its excess margin $b_m(\lambda,\delta_m)$ vanishes 
uniformly, the selected thresholds approach a common 
population limit $\lambda^\star$.
\end{remark}

\begin{remark}[Verifying Assumption~2 via Lemma~B.3]
\label{rem:a2_justification}
Assumption~2 is not an additional hypothesis imposed on the method; it
is a consequence of the standard split-conformal quantile construction
under mild regularity.
Specifically, Lemma~B.3 shows that, if the score
distribution $F_{\hat{\phi}_{\mathrm{tune}}}$ is continuous and has density bounded
below by $c>0$ in a neighbourhood of the population quantile
$q^\star_{\hat{\phi}_{\mathrm{tune}}}$, then the empirical conformal threshold
$\hat{q}_{1-\alpha}$ satisfies
\[
    \mathbb{P}\!\left(
        \bigl|\hat{q}_{1-\alpha}-q^\star_{\hat{\phi}_{\mathrm{tune}}}\bigr|>t
    \right)
    \leq
    2\exp\!\left(
        -2m_{\mathrm{cal}}
        \!\left(ct-\frac{2}{m_{\mathrm{cal}}}\right)_{\!+}^{2}
    \right),
\]
which implies $\hat{q}_{1-\alpha}\xrightarrow{p}q^\star_{\hat{\phi}_{\mathrm{tune}}}$
as $m_{\mathrm{cal}}\to\infty$.
Identifying $\widehat{\lambda}_{\mathrm{DCO}}=\hat{q}_{1-\alpha}$ and
$\lambda^\star=q^\star_{\hat{\phi}_{\mathrm{tune}}}$, Assumption~2 therefore holds
whenever the score distribution satisfies the local density condition
of Lemma~B.3.
The correction term $\bigl(ct-2/m_{\mathrm{cal}}\bigr)_+$ accounts for the
discreteness of the order statistic and is negligible once
$m_{\mathrm{cal}}\gg 1/(ct)$.
\end{remark}

\section{Experimental Details and Additional Results}
\label{app:exp}
This appendix provides full implementation details and 
additional results for the experiments in 
Section~\ref{sec:experiment}. 
Appendix~\ref{app:exp-regression} covers the regression 
experiment and Appendix~\ref{app:exp-classification} 
covers the ImageNet-A classification experiment.
\subsection{Shared Experimental Components}
\label{app:shared}

\paragraph{Matched-budget protocol.}
For all datasets, BQ/CRC uses the union
$D_{\mathrm{tune}}\cup D_{\mathrm{cal}}$ as its
calibration pool, so both DCO-Warmstart and the coupled
baseline consume the same total number of
non-training examples.
DCO-Warmstart allocates these examples across two independent
splits; BQ/CRC treats them as a single pool and
applies its own risk-control procedure.

\paragraph{BQ/CRC implementation.}
The BQ threshold is selected as
\begin{equation}
    \hat{\lambda}_{\mathrm{BQ}}
    = \inf\bigl\{\lambda :
    \mathbb{P}(L^+(\lambda)\le\alpha
    \mid \ell_{1:m}(\lambda))
    \ge 1-\delta\bigr\},
\end{equation}
where $L^+$ is the Dirichlet-MC upper bound on
conformal risk.
Unless stated otherwise, all BQ/CRC runs use
$\delta=0.05$ (confidence $1-\delta=0.95$),
$M=1{,}000$ Dirichlet draws, and maximum loss
bound $B=1.0$.

\subsection{Regression Experiments}
\label{app:exp-regression}

\paragraph{Shared model specification.}
All regression experiments use the same sparse
Bayesian linear regression family:
\begin{equation}
    Y \mid X,\theta,\theta_0,\tau
    \sim \mathcal{N}(X^\top\theta+\theta_0,\,\tau),
\end{equation}
with hierarchical priors
\begin{equation}
    \theta_j \sim \mathrm{Laplace}(0,b),\quad
    b \sim \mathrm{Gamma}(1,1),\quad
    \tau \sim \mathrm{HalfNormal}(c),\quad
    \theta_0 \sim \mathcal{N}(0,10).
\end{equation}
The non-conformity score is the posterior predictive
negative log-likelihood,
\begin{equation}
    S(x,y)=-\log\hat{p}(y\mid x),\quad
    \log\hat{p}(y\mid x)
    =\log\!\left(\frac{1}{T}
    \sum_{t=1}^{T}p(y\mid x,\theta^{(t)})\right),
\end{equation}
where $\theta^{(t)}$ are NUTS MCMC posterior draws.
Prediction intervals are formed on a response grid
spanning $[\min(y_{\mathrm{train}})-2,\,
\max(y_{\mathrm{train}})+2]$.
Inputs $X$ and targets $y$ are standardised with
\texttt{StandardScaler}.
DCO-Warmstart searches over prior scale $c\in\{1.0,0.02\}$;
BQ/CRC uses a fixed structure with $c=1.0$.

\paragraph{Non-Bayesian baselines (Diabetes only).}
\textit{Split~CP.}
A ridge regressor $\hat{f}$ is trained on
$D_{\mathrm{train}}$. Calibration residuals are
$s_i=|y_i-\hat{f}(x_i)|$ for
$(x_i,y_i)\in D_{\mathrm{cal}}$, and the
prediction interval is
$C(x)=[\hat{f}(x)-q,\;\hat{f}(x)+q]$
where $q=\mathrm{Quantile}_{1-\alpha}\{s_i\}$.

\textit{CQR.}
Lower and upper quantile regressors
$\hat{q}_{\alpha/2}$, $\hat{q}_{1-\alpha/2}$
(gradient boosting, 200 estimators, max depth~3)
are trained on $D_{\mathrm{train}}$.
Calibration scores are
\begin{equation}
    s_i = \max\bigl(
        \hat{q}_{\alpha/2}(x_i)-y_i,\;
        y_i-\hat{q}_{1-\alpha/2}(x_i)
    \bigr),
\end{equation}
and the interval is
$[\hat{q}_{\alpha/2}(x)-q,\;
\hat{q}_{1-\alpha/2}(x)+q]$
where $q=\mathrm{Quantile}_{1-\alpha}\{s_i\}$.

\subsubsection{Diabetes}
\label{app:diabetes}

\paragraph{Data splitting.}
Each run partitions $n=442$ observations into
approximate sizes $|D_{\mathrm{train}}|\approx150$,
$|D_{\mathrm{tune}}|\approx112$,
$|D_{\mathrm{cal}}|\approx113$,
$|D_{\mathrm{test}}|\approx67$,
repeated over 50 random seeds.

\paragraph{MCMC settings.}
$T=8{,}000$ posterior samples after warm-up;
response grid size $B=400$.

\paragraph{Hyperparameter selection.}
Table~\ref{tab:regression_hyp} reports calibrated
test performance for both prior scales.
Both values yield nearly identical coverage and
width after conformal recalibration, indicating
that the calibration step absorbs the effect of
prior misspecification on this dataset.

\begin{table}[h]
\centering
\setlength{\tabcolsep}{8pt}
\caption{Prior scale selection on the Diabetes
dataset (mean~$\pm$~std over 50 splits,
target $1-\alpha=0.8$). Conformal recalibration
on $D_{\mathrm{cal}}$ absorbs the effect of prior
misspecification.}
\label{tab:regression_hyp}
\begin{tabular}{lcc}
\toprule
Prior scale $c$ & Coverage & Width \\
\midrule
$c=1.0$  & $0.841\pm0.063$ & $2.051\pm0.238$ \\
$c=0.02$ & $0.843\pm0.061$ & $2.047\pm0.231$ \\
\bottomrule
\end{tabular}
\end{table}

\paragraph{Threshold-selection outcomes.}
Table~\ref{tab:regression_selection} summarises
the $\lambda$-optimisation outcomes across 50
splits. A feasible $\lambda$ satisfying empirical
coverage $\geq 1-\alpha$ was found in every split;
the fallback was never triggered.

\begin{table}[h]
\centering
\setlength{\tabcolsep}{8pt}
\caption{DCO-Warmstart threshold-selection outcomes on
$D_{\mathrm{tune}}$ across 50 splits.}
\label{tab:regression_selection}
\begin{tabular}{lc}
\toprule
Outcome & Count (out of 50) \\
\midrule
Feasible $\lambda$ found        & 50 \\
Fallback (no feasible $\lambda$) & 0  \\
\bottomrule
\end{tabular}
\end{table}

\paragraph{DirectTune diagnostic.}
The tuning-based threshold produces intervals of
identical average width to the calibration quantile
($1.914$ vs.\ $1.914$) but with higher per-split
coverage variance ($0.072$ vs.\ $0.066$ std),
consistent with the optimistic bias of
Remark~\ref{rem:version2-app}.

\subsubsection{California Housing}
\label{app:cal_housing}

\paragraph{Data source and splitting.}
Loaded via \texttt{sklearn.datasets.fetch\_california\_housing};
rows with non-finite values are removed and the
pool is subsampled to $3{,}000$ observations.
After removing $15\%$ as $D_{\mathrm{test}}$
(${\approx}450$), the remaining budget is split
$30\%/30\%/\mathrm{rest}$ into
$D_{\mathrm{tune}}$, $D_{\mathrm{cal}}$
(${\approx}765$ each), and $D_{\mathrm{train}}$
(${\approx}1{,}020$).
BQ/CRC calibrates on the combined pool of
${\approx}1{,}530$ points.

\paragraph{MCMC settings.}
600 warm-up steps, 3{,}000 posterior samples,
1 chain; response grid of 300 points.
Both DCO-Warmstart and BQ/CRC use identical MCMC settings.

\subsubsection{Concrete Compressive Strength}
\label{app:concrete}

\paragraph{Data source.}
Loaded from OpenML
(\texttt{Concrete\_Compressive\_Strength},
version~3); if unavailable, the UCI Excel file
is used as a fallback.
Rows with missing or non-numeric values are
dropped; no subsampling is applied.

\paragraph{Data splitting.}
Same proportional protocol as California Housing:
$15\%$ test, $30\%$ tuning, $30\%$ calibration,
remainder for training.

\paragraph{MCMC settings.}
Both DCO-Warmstart and BQ/CRC use identical posterior
inference settings: 600 warm-up steps, 3{,}000
posterior samples, 1 chain, and a response grid
of 300 points. This matched configuration is used
for all three target coverage levels reported in
Table~\ref{tab:regression_alpha_levels}.

\subsection{Classification Experiments}
\label{app:exp-classification}

\paragraph{Shared backbone and scoring.}
Both ImageNet-A and CIFAR-100 experiments use a
pretrained ResNet-50 as a frozen feature extractor,
with an MC-dropout classification head trained on
$D_{\mathrm{train}}$.
Predictive uncertainty is approximated via $T=20$
stochastic forward passes (model kept in
\texttt{train} mode).
Two nonconformity scores are evaluated for each
candidate:
\begin{align}
    \bar{p}(c\mid x)
        &= \frac{1}{T}\sum_{t=1}^{T}p_t(c\mid x),
    &S_{\mathrm{post}}(x,c)
        &= -\log\bar{p}(c\mid x), \\
    \tilde{p}(c\mid x)
        &= \frac{\sum_t p_t(c\mid x)^2}
               {\sum_t p_t(c\mid x)},
    &S_{\mathrm{aoi}}(x,c)
        &= -\log\tilde{p}(c\mid x).
\end{align}
Both scores induce nested prediction sets
$C_\lambda(x)=\{y:S(x,y)\le\lambda\}$.

\paragraph{Shared DCO-Warmstart candidate class.}
DCO-Warmstart searches over $|\Phi|=16$ configurations:
\begin{itemize}
    \item score type:
          \texttt{posterior\_nll}, \texttt{aoi\_nll};
    \item dropout rate:
          $\{0.05,0.10,0.20,0.30\}$;
    \item hidden widths:
          $\{(512,256),(256,128)\}$.
\end{itemize}
For each candidate the tuning-stage threshold is
searched over 80 quantile-based values derived
from $D_{\mathrm{tune}}$; ties are broken by P95
set size then by $\lambda$ value.
The deployed threshold is recalibrated on
$D_{\mathrm{cal}}$ using the exact split-conformal
quantile.

\paragraph{Shared BQ/CRC baseline.}
The baseline uses a fixed structure
(\texttt{posterior\_nll}, dropout $0.05$, hidden
$(512,256)$) and calibrates on the combined pool
$D_{\mathrm{tune}}\cup D_{\mathrm{cal}}$, with
$\delta=0.05$, $B=1.0$, and $M=1{,}000$ Dirichlet
draws.

\subsubsection{ImageNet-A}
\label{app:imageneta}

\paragraph{Data splitting.}
Data are partitioned into approximately 2{,}000
samples each for $D_{\mathrm{train}}$,
$D_{\mathrm{cal}}$, and $D_{\mathrm{test}}$, and
1{,}000 for $D_{\mathrm{tune}}$, across 50
stratified random seeds (198 classes).
The BQ/CRC calibration pool therefore contains
$3{,}000$ points.

\paragraph{Candidate configurations and selection.}
Table~\ref{tab:imageneta_candidates} lists all 16
configurations evaluated on $D_{\mathrm{tune}}$
for a representative seed.
All configurations achieve the target coverage
($0.8$) on the tuning split; DCO-Warmstart selects
\texttt{cand\_001} (posterior NLL, dropout $0.05$,
hidden $(512,256)$) as it achieves the smallest
average set size ($22.105$) among feasible
candidates.
Table~\ref{tab:imagenet_selection_freq} summarises
selection frequencies across all 50 seeds.

\paragraph{Single-run illustration.}
For the representative seed in
Table~\ref{tab:imageneta_candidates}, the tuning
threshold $\lambda_{\mathrm{tune}}=6.180$ falls
below the nominal target when applied directly;
recalibration on $D_{\mathrm{cal}}$ raises it to
$q_{\mathrm{cal}}=6.419$, restoring coverage to
$0.797$.
BQ selects the more conservative threshold
$\lambda_{\mathrm{BQ}}=6.510$, achieving coverage
$0.808$ at the cost of larger sets
($25.56$ vs.\ $24.56$).

\begin{table}[h]
\centering
\small
\caption{Top 10 candidate configurations evaluated on
$D_{\mathrm{tune}}$ for ImageNet-A (representative
seed). All configurations are feasible; selected
configuration in bold.}
\label{tab:imageneta_candidates}
\resizebox{0.8\linewidth}{!}{%
\begin{tabular}{llccrccc}
\toprule
ID & Score & Dropout & Hidden
   & $\lambda$ & Status & Avg Size & P95 \\
\midrule
\textbf{cand\_001} & posterior\_nll
    & 0.05 & (512,256) & 6.180
    & feasible & \textbf{22.105} & 49.00 \\
cand\_003 & posterior\_nll
    & 0.10 & (512,256) & 6.095
    & feasible & 22.216 & 49.00 \\
cand\_009 & aoi\_nll
    & 0.05 & (512,256) & 5.725
    & feasible & 22.238 & 47.00 \\
cand\_011 & aoi\_nll
    & 0.10 & (512,256) & 5.337
    & feasible & 22.363 & 46.00 \\
cand\_005 & posterior\_nll
    & 0.20 & (512,256) & 6.016
    & feasible & 23.587 & 49.00 \\
cand\_007 & posterior\_nll
    & 0.30 & (512,256) & 5.872
    & feasible & 24.519 & 48.00 \\
cand\_010 & aoi\_nll
    & 0.05 & (256,128) & 5.568
    & feasible & 24.880 & 49.05 \\
cand\_008 & posterior\_nll
    & 0.30 & (256,128) & 5.443
    & feasible & 25.138 & 46.00 \\
cand\_004 & posterior\_nll
    & 0.10 & (256,128) & 6.083
    & feasible & 25.141 & 51.00 \\
cand\_002 & posterior\_nll
    & 0.05 & (256,128) & 6.336
    & feasible & 25.583 & 55.00 \\
\bottomrule
\end{tabular}}
\end{table}

\begin{table}[!htbp]
\centering
\caption{DCO-Warmstart candidate selection frequencies on
$D_{\mathrm{tune}}$ across 50 seeds on ImageNet-A.
All 16 configurations achieved the coverage
constraint in every seed (infeasible fallback:
0/50).}
\label{tab:imagenet_selection_freq}
\setlength{\tabcolsep}{6pt}
\begin{tabular}{llccc}
\toprule
Score type & Hidden dims & Dropout
           & Seeds selected & \% \\
\midrule
\texttt{aoi\_nll}       & $(512,256)$ & $0.05$ & 20 & 40\% \\
\texttt{posterior\_nll} & $(512,256)$ & $0.05$ & 12 & 24\% \\
\texttt{posterior\_nll} & $(512,256)$ & $0.10$ & 7  & 14\% \\
\texttt{aoi\_nll}       & $(512,256)$ & $0.10$ & 4  & 8\%  \\
\texttt{posterior\_nll} & $(512,256)$ & $0.20$ & 4  & 8\%  \\
\texttt{posterior\_nll} & $(256,128)$ & $0.05$ & 1  & 2\%  \\
\texttt{posterior\_nll} & $(256,128)$ & $0.30$ & 1  & 2\%  \\
\texttt{aoi\_nll}       & $(256,128)$ & $0.05$ & 1  & 2\%  \\
\midrule
\multicolumn{3}{l}{\textit{Marginal: score type}} & & \\
\quad\texttt{aoi\_nll}       && & 25 & 50\% \\
\quad\texttt{posterior\_nll} && & 25 & 50\% \\
\multicolumn{3}{l}{\textit{Marginal: hidden dims}} & & \\
\quad$(512,256)$ && & 47 & 94\% \\
\quad$(256,128)$ && & 3  & 6\%  \\
\multicolumn{3}{l}{\textit{Marginal: dropout rate}} & & \\
\quad$0.05$ && & 34 & 68\% \\
\quad$0.10$ && & 11 & 22\% \\
\quad$0.20$ && & 4  & 8\%  \\
\quad$0.30$ && & 1  & 2\%  \\
\bottomrule
\end{tabular}
\end{table}

\begin{figure}[t]
    \centering
    \includegraphics[width=0.72\linewidth]{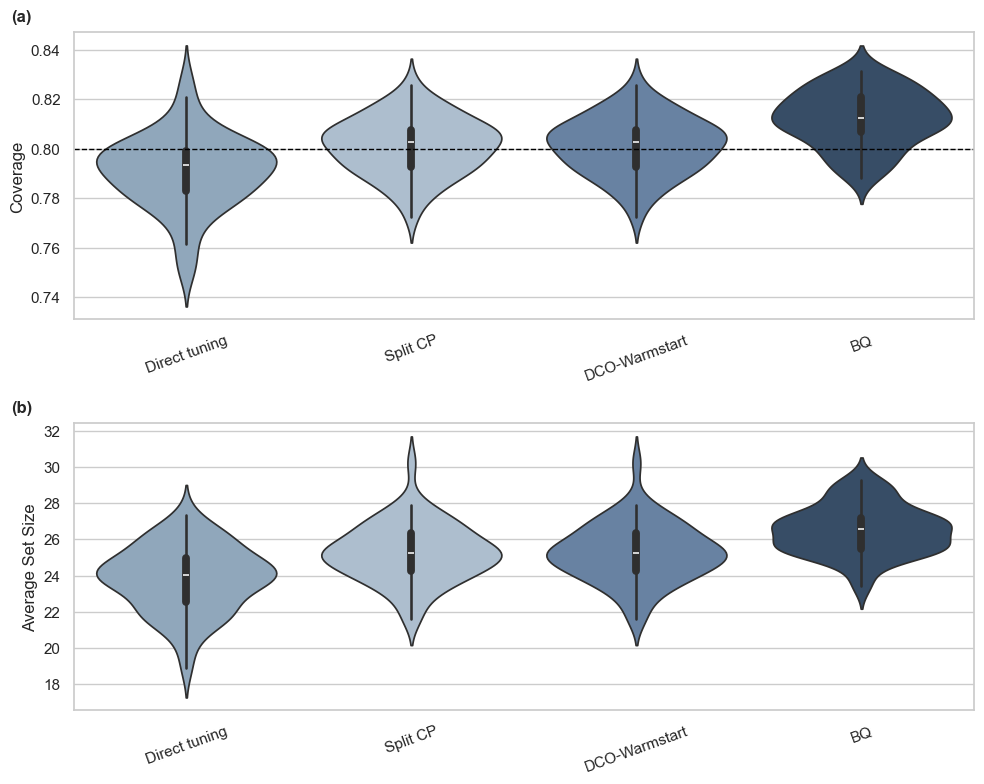}
    \caption{Distribution of test coverage~(a) and
    prediction set size~(b) on ImageNet-A across 50
    random splits. The dotted line marks the target
    $1-\alpha=0.8$. DCO-Warmstart concentrates around the
    target with a tighter set-size distribution than
    BQ.}
    \label{fig:imagenet_hist}
\end{figure}

\begin{figure}[t]
    \centering
    \includegraphics[width=0.52\linewidth]{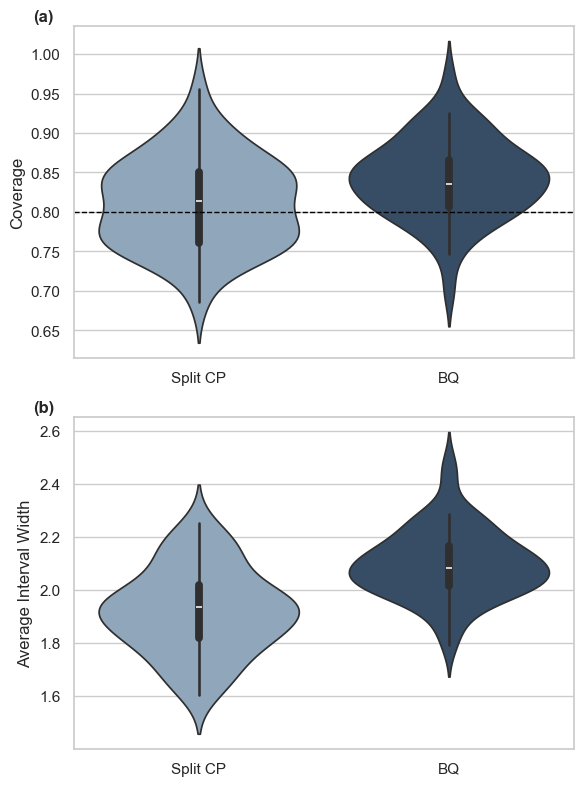}
    \caption{Coverage~(a) and average interval
    width~(b) for Split~CP and BQ on the Diabetes
    dataset. BQ achieves higher coverage but
    produces wider intervals, reflecting the
    conservative bias of coupled threshold selection
    and calibration.}
    \label{fig:aoi_splitcp}
\end{figure}

\begin{figure}[t]
    \centering
    \includegraphics[width=0.52\linewidth]{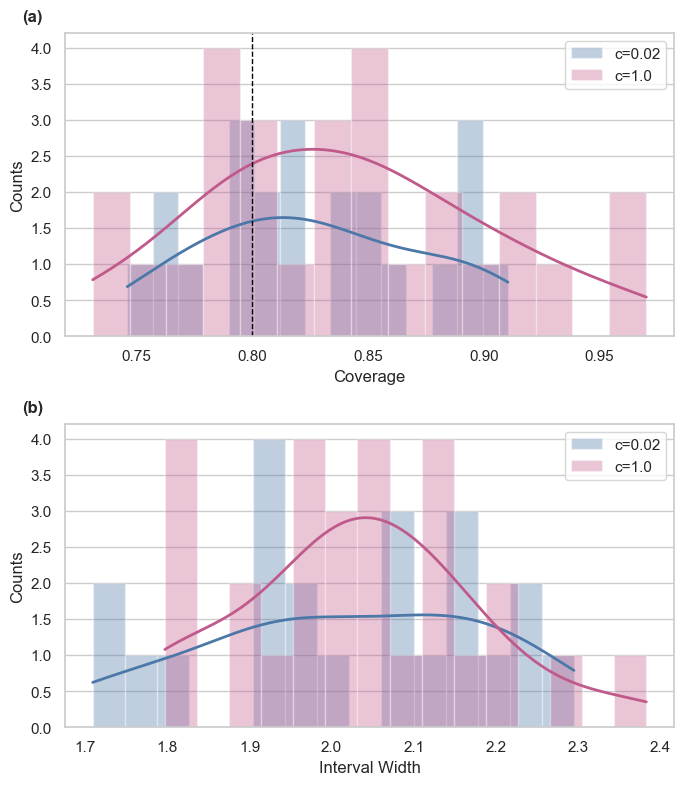}
    \caption{Coverage~(a) and average interval
    width~(b) under prior scales $c=1.0$ and
    $c=0.02$ on the Diabetes dataset across 50
    splits ($1-\alpha=0.8$). Conformal recalibration
    on $D_{\mathrm{cal}}$ absorbs the effect of
    prior misspecification.}
    \label{fig:prior_scale}
\end{figure}

\subsubsection{CIFAR-100}
\label{app:cifar100}

\paragraph{Data source.}
Raw images are loaded from
\texttt{torchvision.datasets.CIFAR100} (training
and test partitions concatenated into a single
pool); a pretrained ResNet-50
(\texttt{ResNet50\_Weights.DEFAULT}) is used as a
frozen feature extractor with the final
classification layer removed.

\paragraph{Classification head training.}
Adam optimiser, learning rate $10^{-3}$, weight
decay $10^{-4}$, batch size $128$, $15$ epochs.

\paragraph{Data splitting.}
For each of 50 random seeds:
$|D_{\mathrm{train}}|=2{,}000$,
$|D_{\mathrm{tune}}|=1{,}000$,
$|D_{\mathrm{cal}}|=2{,}000$,
$|D_{\mathrm{test}}|=2{,}000$.
The BQ/CRC calibration pool therefore contains
$3{,}000$ points, matching the ImageNet-A protocol.

\end{document}